\documentclass{article}

    \PassOptionsToPackage{numbers, compress}{natbib}
 \usepackage[preprint]{neurips_2026}

\usepackage[utf8]{inputenc} 
\usepackage[T1]{fontenc}    
\usepackage{hyperref}       
\usepackage{url}            
\usepackage{booktabs}       
\usepackage{amsfonts}       
\usepackage{nicefrac}       
\usepackage{microtype}      
\usepackage{xcolor}         

\hypersetup{colorlinks=true,linkcolor=blue,citecolor=blue,urlcolor=magenta}

\usepackage{graphicx}
\usepackage{subcaption}
\usepackage{wrapfig}
\usepackage{adjustbox}
\usepackage{multirow}
\usepackage{multicol}
\usepackage[table]{xcolor}
\usepackage{colortbl}
\usepackage{pifont}
\usepackage[most]{tcolorbox}
\usepackage{amsmath}
\usepackage{amssymb}
\usepackage{mathtools}
\usepackage{amsthm}
\usepackage[capitalize,noabbrev]{cleveref}
\usepackage[textsize=tiny]{todonotes}

\definecolor{lightblue}{rgb}{0.2, 0.4, 0.8}
\definecolor{lightgrey}{gray}{0.92}
\newcommand{\greycell}[1]{\cellcolor{lightgrey}{#1}}

\newcommand{\cmark}{\textcolor{green!70!black}{\ding{51}}}
\newcommand{\xmark}{\textcolor{red}{\ding{55}}}

\theoremstyle{plain}
\newtheorem{theorem}{Theorem}[section]

\newtheorem{corollary}[theorem]{Corollary}
\theoremstyle{definition}
\newtheorem{definition}[theorem]{Definition}

\theoremstyle{remark}

\definecolor{modelgray}{RGB}{220,224,230}
\definecolor{aeagray}{RGB}{230,230,230}

\definecolor{goodgreen}{RGB}{24,138,67}
\definecolor{badred}{RGB}{200,45,45}

\newcommand{\modelrow}[1]{%
\rowcolor{modelgray}\multicolumn{9}{c}{#1}\\
}

\renewcommand{\greycell}[1]{\cellcolor{aeagray}{#1}}

\renewcommand{\cmark}{\textcolor{goodgreen}{\checkmark}}
\renewcommand{\xmark}{\textcolor{badred}{\ding{55}}}

\title{A Sober Look at Agentic Misalignment in \\ Automated Workflows}

\author{%
  \textbf{Wenqian Ye}\textsuperscript{1$\dagger$}, 
  \textbf{Bo Yuan}\textsuperscript{2}, 
  \textbf{Zhichao Xu}\textsuperscript{3}, 
  \textbf{Yijun Tian}\textsuperscript{4}, 
  \textbf{Yawei Wang}\textsuperscript{5}, \\
  \textbf{Henry Kautz}\textsuperscript{1}, 
  \textbf{Aidong Zhang}\textsuperscript{1$\dagger$} \\[0.5em]
  \textsuperscript{1}University of Virginia  \quad
  \textsuperscript{2}Georgia Institute of Technology  \quad
  \textsuperscript{3}University of Utah \\
  \textsuperscript{4}University of Notre Dame \quad
  \textsuperscript{5}George Washington University \\
  \texttt{$^\dagger$\{wenqian, aidong\}@virginia.edu}
}

\begin{document}

\maketitle

\begin{abstract}
We study a class of emergent misalignment in multi-agent systems (MAS), with a focus on automated workflows, which we refer to \textit{agentic misalignment}. Although these systems can solve complex tasks, they often fail because agents act according to implicit proxy utilities that do not align with the intended human goals. We formally define these behaviors and analyze them within a Bayesian framework, showing that generic utilities naturally lead to \textit{posterior collapse} of agents in automated workflows. To address this issue, we propose Agentic Evidence Attribution (AEA), a novel alignment paradigm that improves agent posteriors using context-specific evidence. AEA reasons over agent actions and provides structured evidence to correct misaligned behavior during collaboration. To better understand the role of evidence, we study two instantiations of AEA: self-reflection (internal evidence from the model) and weak-to-strong generalization (external evidence on
the agentic trajectory). We show that a small evidence model effectively aligns the MAS by providing orthogonal failure attribution. Our results clarify the sources of agentic misalignment in automated workflows and show that evidence-based alignment can effectively improve agent collaboration and leads to reliable multi-agent systems built on automated workflows.
\end{abstract}

\section{Introduction}

Multi-agent systems (MAS) with Large Language Models (LLM) are increasingly used for tasks that require decomposition \citep{prasad2024adapt,erdogan2025planandact}, tool use \citep{qin2024toolllm,yuan2025easytool}, and multi-step reasoning \citep{aksitov2024rest,plaat2025multi}. Most recent MAS leverage an orchestrator (or meta-agent) to construct various roles of agents and arrange the order of interactions through a technique called automated workflow generation \citep{li2024autoflow,zhangaflow}. This technique produces a general solution to build collaboration between agents to autonomously plan and act on individual tasks. Although these workflows allow agents to handle complex objectives, recent studies have also reported that MAS often fails in several systematic ways \citep{cemri2025multi,han2024llm,kim2025towards,lynch2025agentic}. These alignment failures arise even when each individual agent performs well in isolation, suggesting an emerging underlying coordination problem within the agentic workflow.

One qualitative perspective for interpreting these misalignments is that agents act according to a shared utility (or reward in some literature) shaped by generic pretraining and post-training, rather than the role-specific goals required by the workflow (Left in Figure \ref{fig:illustration}) \citep{korbak2022rl,ruadulescu2020multi,bengio2025superintelligent}. In automated workflows, agents observe only partial information about the generic objective and the behavior of other agents. As a result, the agent forms an implicit utility that is shaped by a generic training distribution rather than the intended role with respect to the agents themselves. This implicit utility can diverge from the objective of the whole workflow and lead to reward hacking \citep{skalse2022defining}, including misaligned behaviors such as effort minimization, incomplete hand-offs, and short-term decisions that potentially harm other agents' goals \citep{weng2024rewardhack,guo2024large}. These harmful behaviors make automated workflows difficult to control and limit their reliability in deployment.

\begin{wrapfigure}{r}{0.55\textwidth}
    \vskip -10pt
    \centering
    \includegraphics[width=0.55\textwidth]{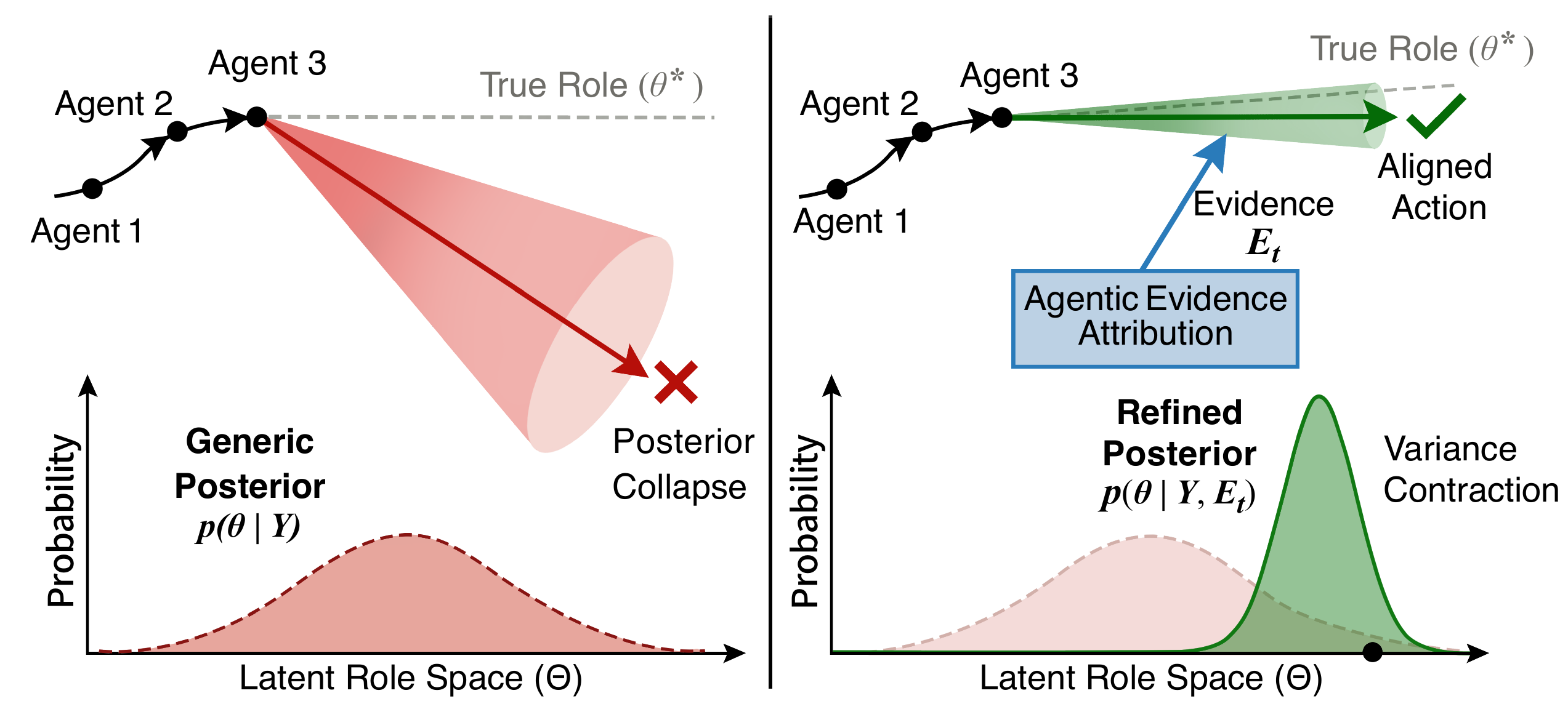}
    \caption{\textbf{Left:} Agents in the workflow suffer from \textit{posterior collapse}, where the generic posterior remains uncertain and fails to identify true role $\theta^*$, leading to system failure. \textbf{Right:} AEA injects structured evidence $E_t$ to induce variance contraction. This sharpens the agent's belief into an evidential posterior, ensuring the agent aligns with specific role.}
    \label{fig:illustration}
    \vskip -8pt
\end{wrapfigure}
To get a more rigorous understanding on the mechanistic of these behaviors, we systematically study this problem by modeling the multi-agent collaboration process as Bayesian inference over latent utility. We hypothesize that each individual agent infers a posterior utility based on its own evidence (from workflow system prompts, intermediate agent states, and observation of other agents) during workflow execution. When the evidence is sparse or generic, the agent's inferred utility will deviate from the role-specific objective, producing inconsistent behavior across steps. This perspective connects MAS failures to a form of utility misspecification: the workflow produces an incomplete likelihood model, and generic alignment training produces a dominant prior that does not separate the agents' roles. Both effects increase the chance of misalignment in workflows. 

Based on these insights, we propose \textit{Agentic Evidence Attribution} (AEA) to enhance evidence for better agentic posterior, a novel alignment paradigm that strengthens the inferred utility of each agent in additional structured evidence extracted from the workflow (shown in Figure \ref{fig:illustration} right). AEA analyzes turn-level traces from existing MAS trajectories, identifies which actions support or harm the task, and assigns context-aware and role-specific feedback. This evidence reduces the ambiguity in the agent's utility posterior and improves coordination. Because AEA operates on workflow traces, it serves as a flexible, model-agnostic framework that can align proprietary multi-agent systems without requiring access to their internal representations.

We evaluate a diverse set of automated multi-agent workflows on a wide range of challenging human-level tasks, including programming, data analysis, scientific reasoning, and competition-level mathematics. We study two instantiations of the proposed framework: \textit{Self-Reflection}, where the MAS use the base model to extract evidence from its own traces, and \textit{Weak-to-Strong} generalization, where a separate trained evidence model is used for agentic failure attribution and alignment. Across various models and benchmarks, AEA consistently reduces coordination failures and improves reliability relative to vanilla multi-agent baselines. Notably, the gains are not explained by additional computation or naive test-time scaling, but by correcting role-level decision errors that would otherwise propagate through the workflow. These results support the view that many multi-agent failures stem from utility misspecification rather than lack of capability, and that evidence-conditioned alignment is an effective mechanism for improving MAS automated workflows.

\vspace{-0.05in}
\section{Related Works}

\noindent \textbf{Agentic Misalignment.}
Misaligned behaviors in reinforcement learning arise when an agent exploits an imperfect reward function to maximize returns without achieving the human's intended goal \citep{wang2025persona}. This reward hacking phenomenon deeply aligns with the well-known Goodhart's Law \citep{karwowskigoodhart} and has been widely observed in LLM alignment, where preference datasets encode spurious cues such as verbosity or sycophancy that models can exploit \citep{yerectifying}. Recent work highlights that agentic systems also inherit these risks: goal misspecification and goal misgeneralization can cause agents to deviate from human intent under distribution shifts, producing systematic failures as capabilities grow \citep{bengio2025superintelligent,sohldickstein20230309}. These observations motivate a more careful study of reward-driven behavior in multi-agent settings. 

\noindent \textbf{Automated Workflows.}
Modern agentic systems rely on automated workflow generation to assign roles and determine the sequence of agent actions. These workflows are produced through prompting or high-level planners (i.e., orchestrators/meta-agents), which implicitly induce a posterior over each agent's utility \citep{mackay}. However, this posterior could be unstable: small prompt changes or planning errors can cause agents to minimize effort, pass insufficient information, or misinterpret their responsibilities \citep{cemri2025multi}. Because the underlying reward is generic rather than role-specific, agents optimize for signals that do not match the workflow's true intent, leading to reward hacking across the steps of the workflow. This gap between automatically generated roles and the generic utility that agents optimize remains a central barrier to trustworthy and reliable multi-agent systems. In this paper, we focus mainly on the \textbf{centralized MAS} discussed in~\citet{kim2025towards} where we have an orchestrator to generate the subagents, since it represents the main stream of current deployed MAS.

\noindent \textbf{Reward Modeling for Agents.}
Reward models (RMs) play a critical role in aligning LLMs with human preferences. Recent work frames alignment as Bayesian inference, where an RM provides evidence that shapes the learned posterior \citep{korbak2022rl,wu2025optimas}. Existing RMs, whether scalar or generative, assume a single shared objective and do not account for the role-specific structure of multi-agent workflows, which limits their ability to prevent coordination failures. Our work addresses this gap by producing context-aware and agent-specific signals. One related work \citep{poddar2024personalizing}, namely Variational Preference Learning (VPL) infers latent user context to model personalized preferences. However, VPL requires architecture-level integration (training specific encoders/decoders) to learn the latent distributions of LLMs, whereas AEA is workflow-agnostic and operates purely in the textual context/prompting space accessible to LLMs. Crucially, our theoretical framework fundamentally reinterprets this dynamic: rather than treating the reward signal as the optimization objective itself, we model it as \textit{external observational evidence} required to contract the variance of the agent's \textit{internal latent utility}, thereby resolving the posterior collapse inherent to pre-trained priors.


\section{A Sober Look at Agentic Misalignment}
\label{sec:theory}

\subsection{Problem Formulation}

We formalize the automated workflow as a partially observed multi-agent Markov Decision Process (POMDP) \citep{spaan2012partially,zhang2025which} augmented with latent role variables. The system is defined by the tuple $\mathcal{M} = \langle \mathcal{N}, \mathcal{S}, \mathcal{A}, \mathcal{T}, \Theta,   \mathcal{U} \rangle$. Here, $\mathcal{N} = \{1, \dots, N\}$ represents the set of agent roles, while $\mathcal{S}$ denotes the global state space that consists of the task context, dialogue history, and tool outputs. The joint action space is given by $\mathcal{A} = \times_{i \in \mathcal{N}} \mathcal{A}_i$, where $\mathcal{A}_i$ is the specific action space for agent $i$. $\mathcal{T}: \mathcal{S} \times \mathcal{A} \to \Delta(\mathcal{S})$ denotes the state transition function with $\Delta(\mathcal{S})$ as the probability simplex over states. To define role-specific behaviors, we introduce $\Theta = \{\theta_1, \dots, \theta_K\}$ as a finite set of \textit{latent utility types} (e.g., ``strict verifier'', ``accurate approximator''). Lastly, $\mathcal{U}: \mathcal{S} \times \mathcal{A} \times \Theta \to \mathbb{R}$ represents the utility function conditioned on the roles.

At any time step $t$, a single agent $i = \psi(t)$ is active according to the workflow orchestration. The agent does not observe the true role intention directly. Instead, the ground-truth role assignment is a random variable $\boldsymbol{\theta}_i \in \Theta$, following a prior distribution $P(\boldsymbol{\theta})$. The \textit{oracle} (ideal) policy maximizes the true utility associated with this latent type:
\begin{equation}
    \pi^*_i(s_t; \theta) \in \arg\max_{a \in \mathcal{A}_i} \; \mathcal{U}(s_t, a; \theta).
\end{equation}

\begin{definition}[Agentic Misalignment via Decisive Errors]
    Let $\tau$ be a trajectory and $Z(\tau) \in \{0,1\}$ be a binary failure indicator ($0$ means no failure). A time step $t$ is a \textit{Decisive Error} if $Z(\tau)=1$, but there exists an action $\tilde{a}_t \in \mathcal{A}_i$ such that the modified trajectory $\tilde{\tau}$ satisfies $Z(\tilde{\tau})=0$. \textit{Agentic Misalignment} occurs when the agent's policy $\hat{\pi}_i$ selects $a_t$ (the error) instead of $\tilde{a}_t$ because $a_t$ maximizes the expected utility under the \textit{generic} posterior, distinct from the specific role type $\boldsymbol{\theta}_i$:
    \begin{equation}
        a_t \in \arg\max_{a} \mathbb{E}_{\theta \sim P(\cdot \mid Y)} [\mathcal{U}(s_t, a; \theta)],
    \end{equation}
     where $a_t \neq \pi^*_i(s_t; \boldsymbol{\theta}_i)$ and $Y$ is the evidence of the agent.
\end{definition}

\subsection{Misalignment as a Role Inference Problem}

The core obstacle is that agent $i$ should infer the specific instantiation $\boldsymbol{\theta}_i$ from the available evidence $Y_t$ at time $t$. We model this as Bayesian inference where the agent maintains a belief state $b_t(\theta) = P(\boldsymbol{\theta}_i = \theta \mid Y_t)$. The failure of multi-agent systems stems from \textit{posterior collapse}, where the generic pre-training prior dominates the weak signal provided by the workflow prompt.

We quantify the divergence between two distributions using the Total Variation (TV) distance, $\| P - Q \|_{\mathrm{TV}} = \frac{1}{2}\sum |P(x)-Q(x)|$. To rigorously bound the posterior update, we introduce a stability condition on the marginal likelihood (evidence).

\begin{theorem}[$\epsilon$-Stability of Role Posteriors]
\label{thm:role-collapse}
    Let $i, j \in \mathcal{N}$ be two distinct roles. Assume the priors are $\epsilon_{\pi}$-close: $\| P(\cdot \mid i) - P(\cdot \mid j) \|_{\mathrm{TV}} \leq \epsilon_{\pi}$ and the likelihoods are $\epsilon_{\ell}$-close. Furthermore, assume the workflow evidence $Y$ is sufficiently informative such that the marginal likelihood is lower-bounded by $\zeta > 0$ for both roles. Then, the posterior distributions are $\delta$-close, where $\delta = \kappa (2\epsilon_{\pi} + \epsilon_{\ell})$ and $\kappa \propto \zeta^{-1}$. Consequently, the probability that agents $i$ and $j$ select distinct actions is upper-bounded:
    \begin{equation}
        P(\hat{\pi}_i(Y) \neq \hat{\pi}_j(Y)) \leq \delta.
    \end{equation}
    If the oracle requirements differ ($\pi^*_i \neq \pi^*_j$), the probability of misalignment is lower-bounded by $1 - \delta$.
\end{theorem}

This result implies that, without distinct external evidence, agents inevitably collapse toward a mean generic behavior. To break this symmetry, we should introduce a principled mechanism that explicitly conditions the posterior on role-specific artifacts. This motivates our proposed framework. 

    
    

\section{Agentic Evidence Attribution}
\label{sec:aea}
In this section, we present Agentic Evidence Attribution (AEA), an analytical and practical framework for correcting agentic misalignment. We first establish an analytical framework in Section~\ref{sec:analytical_framework}, defining the limits of alignment via information theory. Section~\ref{sec:posterior_contraction} details the mechanism of posterior contraction, proving how evidence reduces variance in the agent's latent belief in its role. Finally, Section~\ref{sec:scalable_aea} proposes scalable instantiations of AEA using learned evidence models to satisfy these theoretical conditions in practice.

\subsection{An Analytical Framework}
\label{sec:analytical_framework}

We formalize Agentic Evidence Attribution (AEA) as the injection of the \textbf{extraction of latent evidence} from the system's own history. AEA leverages the information asymmetry between the runtime evidence $Y_t$ from partial observation and the complete computational trace $\tau_{<t}$.

Let $E_t = \mathcal{F}(\tau_{<t})$ be an additional evidence derived from the trajectory history using a function $\mathcal{F}$. The agent's evidence transforms from $Y_t$ to $Y'_t = (Y_t, E_t)$. We analyze the theoretical limits of this intervention using information theory. Let $L^* = \pi^*_i(s_t; \boldsymbol{\theta}_i)$ be the discrete label of the optimal action.

\begin{theorem}[Lower Bound on Misalignment Error]
\label{thm:fano}
    Assume the action space has cardinality $|\mathcal{A}_i| = M$. Let $\hat{a}(Y)$ be any decision rule conditioned on evidence $Y$. The probability of decisive error is lower-bounded by Fano's Inequality:
    \begin{equation}
        P_e \geq 1 - \frac{I(L^*; Y) + \log 2}{\log M},
    \end{equation}
    where $I(L^*; Y)$ is the mutual information between the true optimal action label and the evidence.
\end{theorem}

This lower bound highlights a fundamental limit: no alignment algorithm can succeed without sufficient mutual information between the evidence and the optimal action. This motivates the design of AEA to explicitly maximize this information channel.


\begin{corollary}[Information Gain Condition]
\label{cor:info_gain}
    Let $P_e(Y)$ and $P_e(Y')$ be the probabilities of decisive error under baseline and AEA evidence, respectively. A necessary condition for AEA to strictly reduce the probabilities on misalignment (i.e., $P_e(Y') < P_e(Y)$) is 
    \begin{equation}
        I(L^*; E_t \mid Y_t) > 0
    \end{equation}
\end{corollary}

Since the base model often suffers from $I(L^*; Y) \approx 0$ due to posterior collapse, this condition necessitates introducing an external parameterization $\phi$ such that the mutual information $I(L^*; Y') > 0$. Consequently, the success of alignment depends on the quality of the extraction function $\mathcal{F}$. If $\mathcal{F}$ yields uninformative evidence, misalignment persists regardless of the base model's capabilities, necessitating a learned approach for evidence extraction.


This condition clarifies that AEA is not a manual heuristic. It represents a valid information channel only if the extraction function $\mathcal{F}$ captures correlations between past traces and current optimal actions that are invisible in the baseline prompt.

\subsection{Mechanism via Posterior Contraction}
\label{sec:posterior_contraction}

We now instantiate the utility function $\mathcal{U}$ as a linear model to demonstrate the mechanism of AEA concretely. Let $\phi(s, a) \in \mathbb{R}^d$ be a feature vector of the state-action pair. Let the latent type $\boldsymbol{\theta}$ correspond to a weight vector $\mathbf{w} \in \mathbb{R}^d$, such that $\mathcal{U}(s, a; \mathbf{w}) = \mathbf{w}^\top \phi(s, a)$.

We model the agent's prior belief as a Gaussian: $\mathbf{w} \sim \mathcal{N}(\boldsymbol{\mu}_0, \Sigma_0)$. The baseline evidence $Y$ updates this to a posterior $\mathcal{N}(\boldsymbol{\mu}_Y, \Sigma_Y)$. We model the AEA evidence $E_t$ as a noisy linear observation of the true utility weights (e.g., passing a unit test confirms alignment with specific utility dimensions):
\begin{equation}
    E_t = \mathbf{H} \mathbf{w} + \boldsymbol{\eta},   \boldsymbol{\eta} \sim \mathcal{N}(0, \mathbf{R}),
\end{equation}
where $\mathbf{H}$ projects the latent utility onto observable verification constraints. and $\mathbf{R}$ is the covariance of observation noise.

\begin{theorem}[Covariance Contraction under Verifiable Evidence]
\label{thm:covariance}
    Under the linear-Gaussian assumption, the posterior covariance $\Sigma_{Y'}$ after observing AEA evidence $E_t$ satisfies the Loewner partial order $\Sigma_{Y'} \preceq \Sigma_Y$. Specifically, the precision matrix increases by the Fisher information of the evidence:
    \begin{equation}
        \Sigma_{Y'}^{-1} = \Sigma_Y^{-1} + \mathbf{H}^\top \mathbf{R}^{-1} \mathbf{H}.
    \end{equation}
\end{theorem}

As visualized in Figure \ref{fig:illustration}, this covariance contraction narrows the posterior distribution (green curve), significantly reducing the probability mass assigned to misaligned actions compared to the generic posterior.

Mathematically, this confirms that verifiable evidence reduces the variances in the agent's latent belief space. By injecting $E_t$, we sharpen the distribution around the true role parameter, reducing the likelihood of drawing a misaligned utility vector.

\begin{corollary}[Reliability Improvement via Variance Reduction]
\label{cor:error_reduction}
    Consider two actions $a_1, a_2$ with feature difference $\delta = \phi(s, a_1) - \phi(s, a_2)$ under state $s$. If $a_1$ is optimal ($\mathbf{w}^\top \delta > 0$) and the evidence $E_t$ is unbiased (does not decrease the margin mean $\mu$), then $P_e(Y') \leq P_e(Y)$.
\end{corollary}

This corollary provides the probabilistic guarantee for AEA. By reducing the posterior variance, we strictly increase the probability of selecting the correct role-aligned action. The following sections detail how we implement this via learned evidence models. The proofs of all the theorems and corollaries are given in the Appendix Section~\ref{sec:proofs}.

\subsection{Scalable Instantiation via Learned Evidence}
\label{sec:scalable_aea}

The theoretical guarantees of AEA are based on the existence of an extraction function $\mathcal{F}$ that satisfies the Information Gain Condition (Corollary \ref{cor:info_gain}). Crucially, this evidence $E_t$ is not defined by manual rules. Instead, we define AEA as a learned parameterized function $\mathcal{F}_\phi(\tau_{<t})$. We propose two instantiations to implement this function, motivated by the limitations of the agent's prior belief.

\paragraph{I. Self-Reflection (Bootstrapping).}
The simplest instantiation approximates $\mathcal{F}$ using the base model itself: $E_t \approx \mathcal{M}_{\text{base}}(\tau_{<t})$. However, this approach is theoretically limited by the Dominant Prior assumption (Theorem \ref{thm:role-collapse}). Since the pre-trained prior $P(\theta)$ is shared across roles and favors generic behaviors, querying the same model often resamples from the same uninformative posterior. If the base model's prior dictates that $P(\theta \mid i) \approx P(\theta \mid j)$, self-reflection frequently fails to break this symmetry, resulting in "hallucinated compliance" where the agent rationalizes its generic behavior rather than correcting it.

\paragraph{II. Weak-to-Strong Generalization.}
To resolve the negative impact of the dominant prior, we introduce a distinct evidence model $\mathcal{M}_{\omega}$ (e.g., a smaller open-sourced model) to parameterize $\mathcal{F}$ with $\omega$. Unlike the base model, $\mathcal{M}_{\omega}$ is not intended to solve the task but is trained via reinforcement learning specifically to maximize the mutual information $I(L^*; E_t \mid Y_t)$.
By optimizing $\omega$ on a distribution of successful/failure traces, this model learns to identify the specific latent variables $\theta$ that distinguish roles. These features often ignored by the generic prior of the large model. This ensures that the generated evidence $E_t$ explicitly targets the latent failure modes, transforming alignment from a manual engineering task into a scalable learning problem that strictly satisfies the Information Gain Condition.

\begin{table*}[t]
    \centering
    \caption{Main agentic evaluation on six benchmarks. For each base model, we compare Single Agent, Multi Agent (no alignment), Self-Reflection (AEA), and Weak-to-Strong (AEA). The best performance among two alignment methods for each benchmark is highlighted in \textbf{boldface}.}
    \resizebox{\textwidth}{!}{
    \begin{tabular}{l c c c c c c c c}
    \toprule
    \multirow{2}{*}{\textbf{Setting}} &
    \multirow{2}{*}{\textbf{Aligned}} &
    \multirow{2}{*}{\textbf{AEA Model}} &
    \multicolumn{6}{c}{\textbf{Datasets}} \\
    \cmidrule(lr){4-9}
    &&& \texttt{HumanEval} & \texttt{DataBench} & \texttt{Chemistry} & \texttt{Physics} & \texttt{AIME24} & \texttt{AIME25} \\
    \midrule
    
    \modelrow{\includegraphics[height=1em]{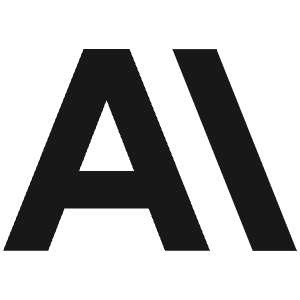} \textbf{Claude 3 Haiku}}
    
    Single Agent & \xmark & -- & 61.6 &  9.3 & 31.7 & 23.5 &  0.0 &  3.3 \\
    Multi Agent  & \xmark & -- & 79.8 & 11.7 & 33.6 & 36.7 & 10.0 &  3.3 \\
    \greycell{Self-Reflection (AEA)} & \greycell{\cmark} & \greycell{Claude 3 Haiku}
    & \greycell{72.0} & \greycell{\textbf{27.3}} & \greycell{31.7} & \greycell{41.2} & \greycell{16.7} & \greycell{13.3} \\
    \greycell{Weak-to-Strong (AEA)} & \greycell{\cmark} & \greycell{AEA-4B}
    & \greycell{\textbf{79.3}} & \greycell{23.9} & \greycell{\textbf{36.6}} & \greycell{\textbf{58.8}} & \greycell{\textbf{20.0}} & \greycell{\textbf{16.7}} \\
    \midrule
    
    \modelrow{\includegraphics[height=1em]{figures/anthropic.png} \textbf{Claude 3.5 Sonnet}}
    
    Single Agent & \xmark & -- & 82.3 & 27.3 & 47.6 & 35.3 & 16.7 &  6.7 \\
    Multi Agent  & \xmark & -- & 84.8 & 32.5 & 51.2 & 58.8 & 23.3 & 16.7 \\
    \greycell{Self-Reflection (AEA)} & \greycell{\cmark} & \greycell{Claude 3.5 Sonnet}
    & \greycell{\textbf{89.0}} & \greycell{29.7} & \greycell{54.9} & \greycell{60.2} & \greycell{23.3} & \greycell{\textbf{26.7}} \\
    \greycell{Weak-to-Strong (AEA)} & \greycell{\cmark} & \greycell{AEA-4B}
    & \greycell{86.6} & \greycell{\textbf{35.0}} & \greycell{\textbf{56.1}} & \greycell{\textbf{61.8}} & \greycell{\textbf{30.0}} & \greycell{\textbf{26.7}} \\
    \midrule
    
    \modelrow{\includegraphics[height=1em]{figures/anthropic.png} \textbf{Claude 3.7 Sonnet}}
    
    Single Agent & \xmark & -- & 84.8 & 15.2 & 61.5 & 58.8 & 23.3 & 13.3 \\
    Multi Agent  & \xmark & -- & 85.4 & 10.5 & 61.0 & 64.7 & 40.0 & 26.7 \\
    \greycell{Self-Reflection (AEA)} & \greycell{\cmark} & \greycell{Claude 3.7 Sonnet}
    & \greycell{89.6} & \greycell{29.3} & \greycell{63.4} & \greycell{67.6} & \greycell{16.7} & \greycell{16.7} \\
    \greycell{Weak-to-Strong (AEA)} & \greycell{\cmark} & \greycell{AEA-4B}
    & \greycell{\textbf{90.9}} & \greycell{\textbf{35.6}} & \greycell{\textbf{71.4}} & \greycell{\textbf{68.4}} & \greycell{\textbf{40.0}} & \greycell{\textbf{46.6}} \\
    \midrule
    
    \modelrow{\includegraphics[height=1em]{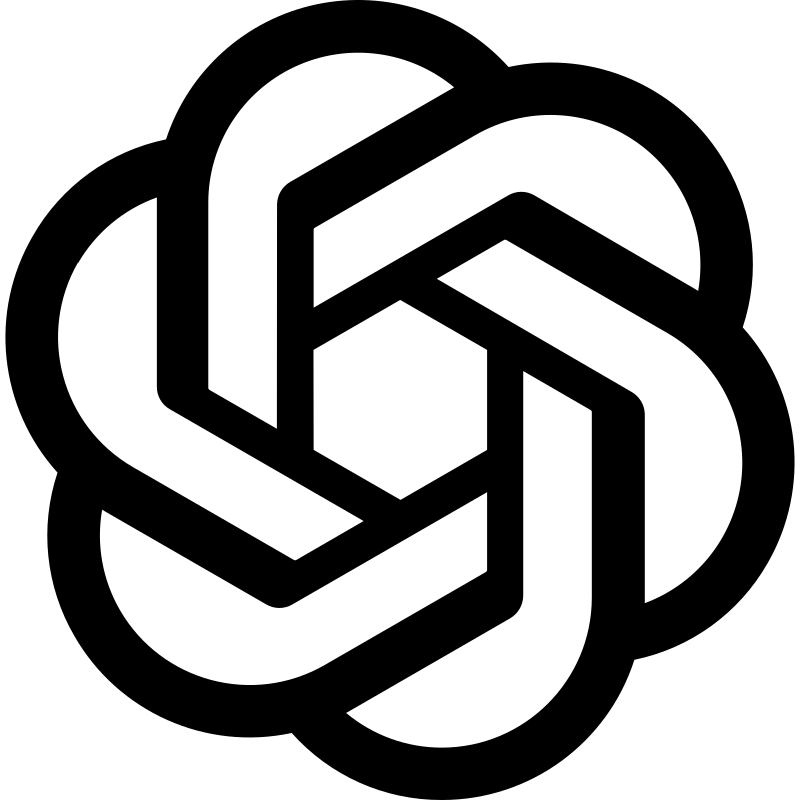} \textbf{GPT-4o-mini}}
    
    Single Agent & \xmark & -- & 91.5 & 26.2 & 61.0 & 67.6 & 13.3 & 13.3 \\
    Multi Agent  & \xmark & -- & 92.0 & 28.8 & 58.5 & 61.8 & 16.7 & 26.7 \\
    \greycell{Self-Reflection (AEA)} & \greycell{\cmark} & \greycell{GPT-4o-mini}
    & \greycell{\textbf{93.3}} & \greycell{27.3} & \greycell{55.9} & \greycell{\textbf{61.8}} & \greycell{6.7} & \greycell{20.0} \\
    \greycell{Weak-to-Strong (AEA)} & \greycell{\cmark} & \greycell{AEA-4B}
    & \greycell{92.2} & \greycell{\textbf{29.7}} & \greycell{\textbf{65.9}} & \greycell{\textbf{61.8}} & \greycell{\textbf{23.3}} & \greycell{\textbf{30.0}} \\
    \midrule
    
    \modelrow{\includegraphics[height=1em]{figures/openai.png} \textbf{GPT-4o}}
    
    Single Agent & \xmark & -- & 95.7 & 27.3 & 73.2 & 82.4 & 16.7 & 10.0 \\
    Multi Agent  & \xmark & -- & 91.5 & 30.7 & 70.7 & 61.8 & 23.3 & 33.3 \\
    \greycell{Self-Reflection (AEA)} & \greycell{\cmark} & \greycell{GPT-4o}
    & \greycell{94.5} & \greycell{27.7} & \greycell{65.9} & \greycell{71.0} & \greycell{30.0} & \greycell{26.7} \\
    \greycell{Weak-to-Strong (AEA)} & \greycell{\cmark} & \greycell{AEA-4B}
    & \greycell{\textbf{96.3}} & \greycell{\textbf{31.2}} & \greycell{\textbf{73.2}} & \greycell{\textbf{76.5}} & \greycell{\textbf{40.0}} & \greycell{\textbf{33.3}} \\
    \bottomrule
    \end{tabular}
    }
    \label{tab:main}
    \end{table*}
    
\section{Experiments}

\subsection{Agentic Trace Generation}
\label{sec:agent-data}
To learn the generalizable agentic trace distribution of the MAS, we construct agentic trace reasoning data from four challenging agent benchmarks shown in the Appendix, Table~\ref{tab:datasets_construct}: \texttt{GAIA}~\citep{mialon2023gaia}, \texttt{AssistantBench}~\citep{yoran2024assistantbench}, \texttt{LiveBench}~\citep{whitelivebench}, and \texttt{Who\&When}~\citep{zhang2025which}. The first three benchmarks are originally designed for single-agent evaluation and focus on various challenging tasks such as coding, reasoning, tool use, and web interaction. In contrast, \texttt{Who\&When} is a native multi-agent benchmark that targets failure attribution in LLM-based multi-agent systems. However, the current scale of these benchmarks is limited and insufficient on its own to learn generalizable agentic evidence signals and to reflect the real distributional behavior of agents. We therefore leverage all four benchmarks to construct a unified agentic reasoning dataset by running them under automated workflows and collecting annotated execution traces.

For each task, we generate the trajectories in the json format. We first use Claude-4 Opus \citep{claude3} to produce initial pseudo-annotations that identify the failing agent, the decisive step, and a candidate repair action. Similar to \citet{zhang2025which}, these annotations are then reviewed by a team of 5 human experts, who verify the correctness of the failure attribution, adjust the identified steps when necessary, and rewrite prompt optimization suggestions to ensure they are executable and role-consistent. The final dataset contains human-verified agentic traces with structured evidence labels, which are used to train the AEA evidence model. Details of the MAS implementation are provided in Section~\ref{sec:imp_details}.

\subsection{Evaluation}
To evaluate the performance of both single-agent and multi-agent settings, we use six human-level challenging benchmarks that span a wide range of difficult tasks. \texttt{HumanEval} \citep{Chen2021EvaluatingLL} tests the generation of functional code with human verification. \texttt{DataBench} \citep{grijalba2024question} focuses on tool use and decision making in realistic tabular data analysis. In \texttt{DataBench}, we only provide the link for the tabular data, which significantly increases the difficulty. SciBench \citep{wang2024scibench} provides two scientific reasoning tracks in \texttt{Chemistry} and \texttt{Physics} that require accurate, grounded problem solving. \texttt{AIME24} \citep{huggingface_aime_2024} and \texttt{AIME25} \citep{huggingface_aime_2025} measure mathematical reasoning with competition-level difficulty. For performance evaluation, we use LLM-as-a-Judge with GPT-4o \citep{achiam2023gpt} to compare ground truth answers and multi-agent results. The evaluation prompts are provided in the Appendix~\ref{sec:prompts}.

\subsection{Implementation Details}
\label{sec:imp_details}

\noindent \textbf{Multi-Agent System.}
We implement MAS using CaptainAgent from the AG2 library \citep{wu2024autogen} for automated workflow generation. CaptainAgent constructs workflows by assigning roles to LLM agents and coordinating their turn-by-turn interactions through a shared memory and tool interface. Each agent receives a role description, observes the partial states of the system, and produces an action that updates the workflow. This design makes CaptainAgent a representative automated workflow system with the support of modular agent roles, sequential execution, and tool-based actions. This setup provides a natural testbed for studying misalignment in multi-agent workflows.

\noindent \textbf{Weak-to-Strong Training.}
We train AEA-4B by optimizing a reasoning model (using Qwen3-4B as the backbone) on multi-agent workflow traces via a two-stage procedure. Each training example contains the full workflow context (roles, turn-level dialog, tool calls, and results) and a target evidence output that specifies the failing agent, the decisive step, and a concrete correction. In the first stage, we apply a supervised warm start on the RM-R1~\citep{chen2025rm} dataset to align the model with standard preference representations. In the second stage, we apply GRPO-based reinforcement learning with verifiable rewards~\cite{shao2024deepseekmath}. The reward function is decomposed into four weighted components: (1) \textit{agent identification} (40\%), checking for exact name matches; (2) \textit{rating alignment} (30\%), penalizing deviations from ground-truth failure severity; (3) \textit{correction validity} (20\%), verifying that proposed fixes are non-trivial; and (4) \textit{reasoning completeness} (10\%), encouraging comprehensive failure attribution. Invalid JSON formats receive a fixed penalty of $-0.8$, with the total reward clamped to $[-1, 1]$. We set the learning rate to $1 \times 10^{-6}$ and the number of group rollouts to $G=7$. We maximize the GRPO objective:
\begin{align}
    \mathcal{J}(\omega) = \frac{1}{G} \sum_{i=1}^G   [ \min  ( \rho_i(\omega) A_i, \text{clip}(\rho_i(\omega), 1-\varepsilon, 1+\varepsilon) A_i  )  - \beta \mathbb{D}_{\text{KL}}(\mathcal{M}_{\omega} || \mathcal{M}_{\text{ref}})  ]
\end{align}
where $\mathcal{M}_{\omega}$ denotes the evidence model parameterized by $\omega$, $\rho_i(\omega) = \frac{\mathcal{M}_{\omega}(y_i|x)}{\mathcal{M}_{\text{old}}(y_i|x)}$ is the probability ratio with the old model policy, and $A_i$ is the group-relative advantage. We use a clipping range of $\varepsilon = 0.2$ and a KL penalty $\beta = 0.001$, relative to reference model $\mathcal{M}_{\text{ref}}$. The maximum context length is set to 12,288 for prompts and 6,144 for responses. We train 5 epochs with fixed prompts and a shared output schema to ensure the evidence model generalizes uniformly across different workflows.

\begin{table*}[!t]
\centering
\caption{Failure Attribution of LLM Multi-Agent Systems on Who\&When Dataset \citep{zhang2025which}. \textbf{Left:} All at Once. \textbf{Right:} Step by Step. Note that we only use the \textbf{All at Once} setting in our proposed AEA alignment method.}
\begin{minipage}{0.49\textwidth}
\centering
\resizebox{0.95\linewidth}{!}{
\begin{tabular}{l cc}
\toprule
\multirow{2}{*}{\textbf{Model}} & \multicolumn{2}{c}{\textbf{All at Once}} \\
\cmidrule(lr){2-3}
& \textbf{Agent Acc. (\%)} & \textbf{Step Acc. (\%)} \\
\midrule
\includegraphics[height=1em]{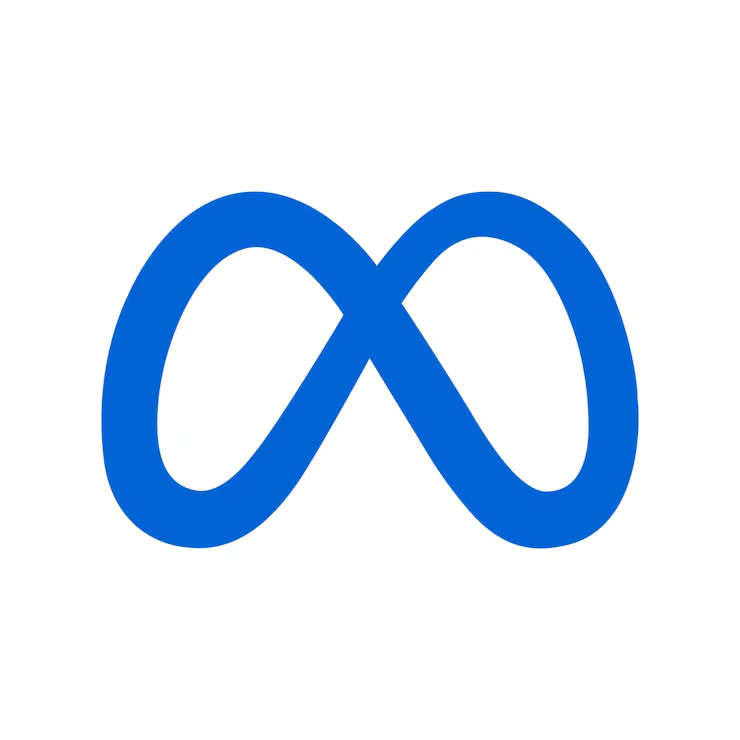}\ Llama-3.1-8B   & 9.52  & 2.38 \\
\includegraphics[height=1em]{figures/meta.png}\ Llama-3.1-70B  & 48.41 & 16.67 \\
\includegraphics[height=1em]{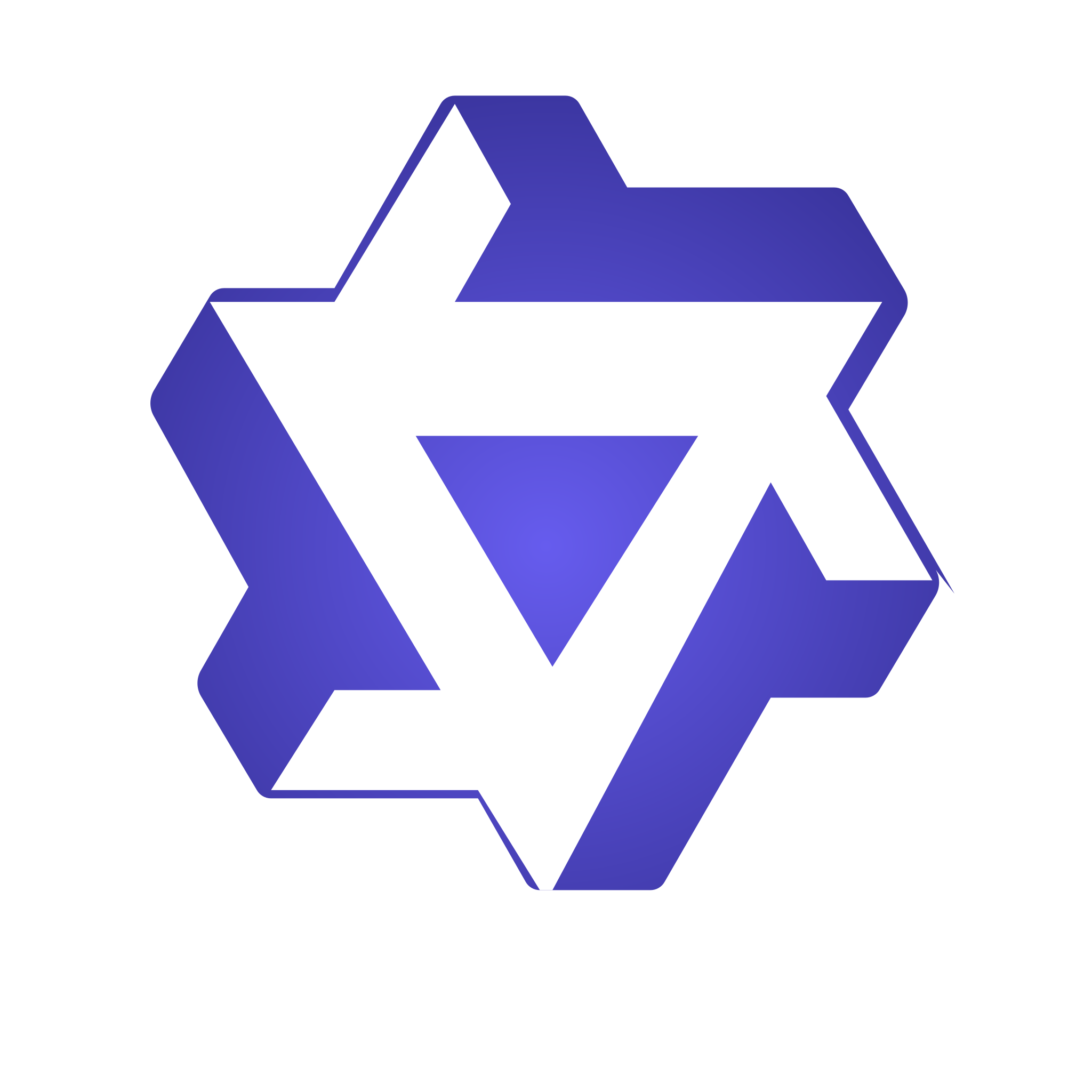}\ Qwen2.5-7B     & 10.32 & 5.56 \\
\includegraphics[height=1em]{figures/qwen.png}\ Qwen2.5-72B    & 55.56 & 10.32 \\
\includegraphics[height=1em]{figures/qwen.png}\ Qwen3-8B       & 31.75 & 5.56 \\
\includegraphics[height=1em]{figures/qwen.png}\ Qwen3-32B      & 53.97 & 15.87 \\
\includegraphics[height=1em]{figures/anthropic.png}\ Claude 3.5 Haiku   & 12.70 & 7.94 \\
\includegraphics[height=1em]{figures/anthropic.png}\ Claude 3.7 Sonnet  & 58.70 & 24.60 \\
\includegraphics[height=1em]{figures/openai.png}\ GPT-4o-mini    & 35.71 & 19.84 \\
\includegraphics[height=1em]{figures/openai.png}\ GPT-4o    & 50.00 & 30.95 \\
\midrule
AEA-4B & \textbf{60.79} & \textbf{32.70} \\
\bottomrule
\end{tabular}
}
\end{minipage}
\hfill
\begin{minipage}{0.49\textwidth}
\centering
\resizebox{0.95\linewidth}{!}{
\begin{tabular}{l cc}
\toprule
\multirow{2}{*}{\textbf{Model}} & \multicolumn{2}{c}{\textbf{Step by Step}} \\
\cmidrule(lr){2-3}
& \textbf{Agent Acc. (\%)} & \textbf{Step Acc. (\%)} \\
\midrule
\includegraphics[height=1em]{figures/meta.png}\ Llama-3.1-8B   & 24.60 & 13.49 \\
\includegraphics[height=1em]{figures/meta.png}\ Llama-3.1-70B  & 30.95 & 19.84 \\
\includegraphics[height=1em]{figures/qwen.png}\ Qwen2.5-7B     & 11.11 & 6.35 \\
\includegraphics[height=1em]{figures/qwen.png}\ Qwen2.5-72B    & 20.63 & 16.67 \\
\includegraphics[height=1em]{figures/qwen.png}\ Qwen3-8B       & 20.63  & 10.32 \\
\includegraphics[height=1em]{figures/qwen.png}\ Qwen3-32B      & 38.10 & 23.02 \\
\includegraphics[height=1em]{figures/anthropic.png}\ Claude 3.5 Haiku   & 21.43 & 16.67 \\
\includegraphics[height=1em]{figures/anthropic.png}\ Claude 3.7 Sonnet  & 46.03 & \textbf{30.95} \\
\includegraphics[height=1em]{figures/openai.png}\ GPT-4o-mini    & 30.95 & 19.05\\
\includegraphics[height=1em]{figures/openai.png}\ GPT-4o   & 33.33  & 26.43\\
\midrule
AEA-4B & \textbf{48.41}  & 25.87 \\
\bottomrule
\end{tabular}
}
\end{minipage}
\vspace{-10pt}
\label{tab:whoandwhen-side-by-side}
\end{table*}

\subsection{Main Results}
\label{sec:main_results}

\noindent \textbf{Inefficient agentic scaling in automated workflows.}
Table~\ref{tab:main} demonstrates that transitioning from a Single Agent to a Multi-Agent workflow generally improves performance, particularly for smaller base models.
For instance, Claude 3 Haiku sees a substantial jump on \texttt{HumanEval} (61.6\% to 79.8\%) and \texttt{Physics} (23.5\% to 36.7\%), suggesting that decomposition helps mitigate reasoning limitations. However, this ``Agentic Test-Time Scaling'' comes at a cost. As illustrated in Figure \ref{fig:time}, the multi-agent system incurs a massive computational overhead, with response times increasing by a factor of 12 to 13. More critically, this scaling exhibits fragility on complex tasks such as \texttt{Chemistry}, where unaligned workflows often yield diminishing returns. We hypothesize that without explicit alignment, increased interactions amplify the probability of ``decisive errors,'' where a single misaligned agent propagates incorrect constraints that poison the collaborative context.

\begin{wrapfigure}{r}{0.48\textwidth}
    \vskip -10pt
    \centering
    \includegraphics[width=0.47\textwidth]{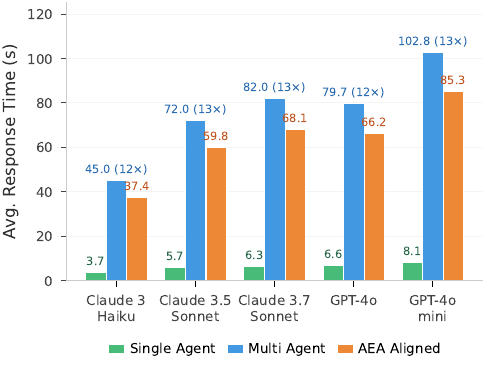}
    \caption{Average response time across agent configurations. Multi-agent systems incur 13$\times$ overhead compared to single-agent baselines.}
    \label{fig:time}
\end{wrapfigure}

\noindent \textbf{AEA effectively enforces role adherence.}
Our proposed method, AEA, consistently improves performance across six benchmarks, effectively reversing the degradation seen in original workflows. The gains are most pronounced in environments that require strict adherence to protocols, such as \texttt{DataBench} and \texttt{AIME}. This empirically supports our theoretical claim that failures are driven by \textit{posterior collapse}. In the absence of AEA, agents default to generic behaviors rather than navigating role-specific constraints. By injecting trace-based evidence, AEA contracts the variance of the agent's utility posterior, reducing role ambiguity and forcing actions to align with structural requirements. Case studies can be found at Appendix~\ref{sec:topo_analysis} and \ref{sec:case}.

\noindent \textbf{Weak-to-Strong breaks the Dominant Prior.}
A key finding is the distinct behavior of Weak-to-Strong generalization compared to Self-Reflection. As shown in Figure \ref{fig:rating_step}, the rating distribution for Self-Reflection is heavily skewed towards high scores (4 and 5), even when the system fails. This empirically validates the Dominant Prior assumption (Theorem \ref{thm:role-collapse}). Because the self-reflecting agent shares the same pre-trained prior as the acting agent, it tends to ``rationalize'' decisive errors rather than correcting them, leading to hallucinated compliance. In contrast, AEA (via Weak-to-Strong) produces a more discriminative distribution of ratings, effectively detecting and mitigating misalignments that the base model ignores.

\noindent \textbf{AEA enables robust error correction.}
We further analyze the impact of AEA in Figure \ref{fig:error_correction} by decomposing the results into three types: (1) \textbf{Fixed} (Incorrect $\to$ Correct), where the alignment successfully repairs a failure; (2) \textbf{Preserved} (Correct $\to$ Correct), where the correct solution is maintained; and (3) \textbf{Regressed} (Correct $\to$ Incorrect), where the alignment breaks a previously correct solution. The results reveal that Self-Reflection suffers from a much higher regression rate, frequently overriding correct reasoning with generic hallucinations. Conversely, AEA (Weak-to-Strong) significantly reduces regression and achieves a higher frequency of \textbf{Fixed} outcomes. This shows that orthogonal evidence from a specialized model is necessary to effectively correct decisive errors without compromising workflow stability.

\subsection{Automated Failure Attribution}
The efficacy of AEA is based on its ability to accurately identify \textit{which agent} and \textit{which step} in a workflow is failing. Table~\ref{tab:whoandwhen-side-by-side} validates this mechanism on the \texttt{Who\&When} benchmark. In ``All at Once'' setting, general-purpose models (even Llama-3.1-70B and GPT-4o) struggle significantly, indicating a lack of sensitivity to causal links in multi-agent failures. Conversely, our specialized AEA-4B model achieves the best accuracy (Step Accuracy of 32.70\% and Agent Accuracy of 60.79\%) on identifying agent failures, outperforming the general-purpose models that are significantly larger. Our method also shows competitive performance on the ``Step by Step'' setting. This result empirically confirms the \textit{Information Gain Condition} (Corollary \ref{cor:info_gain}), showing that specialized training enables AEA to extract the informative evidence required for agentic alignment.

\begin{figure}[!t]
    \centering
    \begin{subfigure}[b]{0.52\linewidth}
        \centering
        \includegraphics[width=\linewidth]{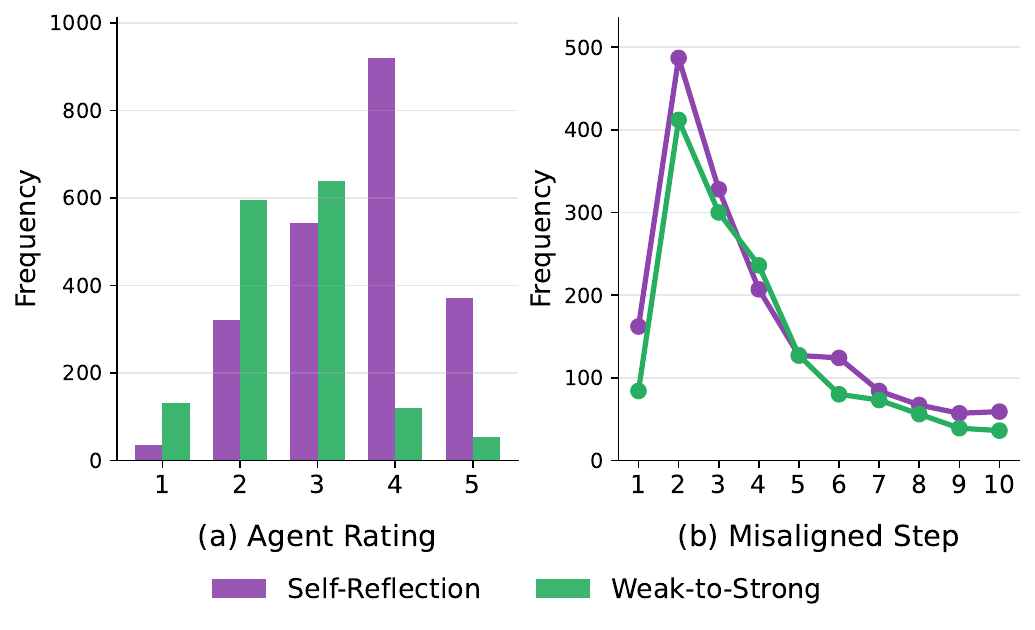}
        \caption{Rating and Step distributions. Self-Reflection skews towards higher ratings, indicating a failure to detect errors due to shared priors.}
        \label{fig:rating_step}
    \end{subfigure}
    \hfill
    \begin{subfigure}[b]{0.45\linewidth}
        \centering
        \includegraphics[width=\linewidth]{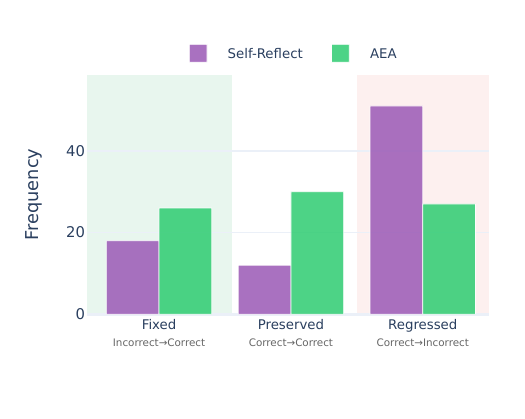}
        \caption{Error Correction Rate. Outcomes are Fixed (Incorrect $\to$ Correct), Preserved (Correct $\to$ Correct), and Regressed (Correct $\to$ Incorrect).}
        \label{fig:error_correction}
    \end{subfigure}
    \caption{Comparison between Self-Reflection and Weak-to-Strong Generalization. \textbf{(a)} Rating and Step distributions reveal that Self-Reflection rationalizes errors with high ratings due to the Dominant Prior. \textbf{(b)} Weak-to-Strong AEA significantly reduces regressions and increases the rate of fixed errors.}
    \label{fig:analysis_combined}
    \vskip -10pt
\end{figure}

\subsection{Empirical Validations on Theoretical Results}
\label{sec:info_gain_exp}

To test the effectiveness of the evidence, we use the decisive-error reduction $\Delta P_e = P_e(Y_t) - P_e(Y_t, E_t)$ as a behavioral surrogate for $I(L^*; E_t \mid Y_t)$ (monotone via Theorem~\ref{thm:fano}) and trace it along an \textit{evidence gradient} on GPT-4o: \textit{Naive Retry} (no new evidence), \textit{Generic Feedback} (role-agnostic message), \textit{Self-Reflection} (same-prior evidence), and \textit{AEA-4B} (external evidence model). As shown in Table~\ref{tab:evidence_gradient}, Naive Retry ($0.002$), Generic Feedback ($0.013$), and Self-Reflection ($0.005$) remain near zero, with Self-Reflection even regressing on \texttt{DataBench} ($-3.0$\,pp) and \texttt{Chemistry} ($-4.8$\,pp), consistent with the Dominant Prior pulling same-prior reflection toward rationalization. AEA-4B reaches $\Delta P_e = 0.071$, an order of magnitude above any alternative, confirming that the Information Gain Condition is the operational mechanism in practice: the gain from AEA is driven by genuinely informative evidence rather than extra compute or generic prompting. The same gradient also yields a practical diagnostic: when AEA repairs a failed trajectory, the bottleneck is \textit{missing evidence}. When it does not, the bottleneck is more likely \textit{missing capability}, indicating whether to invest in stronger base models or in a richer evidence. We discuss the broader efficiency-reliability trade-off and the relation to concurrent works in Appendix~\ref{sec:takeaways}. More empirical validation of theoretical results are provided in Appendix~\ref{sec:additional_qual}.

\begin{table}[!t]
\centering
\caption{Evidence gradient on GPT-4o. $\Delta P_e$ is the average decisive-error reduction across the four benchmarks. Compute alone or generic feedback move $\Delta P_e$ marginally; AEA-4B sits an order of magnitude higher, confirming the Information Gain Condition (Corollary~\ref{cor:info_gain}).}
\label{tab:evidence_gradient}
\resizebox{\textwidth}{!}{
\begin{tabular}{lcc cccc}
\toprule
\textbf{Setting} & \textbf{Evidence Content} & $\Delta P_e$ & \texttt{HumanEval} & \texttt{DataBench} & \texttt{Chemistry} & \texttt{AIME24} \\
\midrule
Multi-Agent (original) & --- & --- & 91.5 & 30.7 & 70.7 & 23.3 \\
Naive Retry & None & 0.002 & 92.1 & 31.5 & 71.2 & 22.0 \\
Generic Feedback & Uninformative & 0.013 & 93.0 & 32.3 & 72.0 & 24.0 \\
Self-Reflection & Moderate (same prior) & 0.005 & 94.5 & 27.7 & 65.9 & 30.0 \\
AEA-4B (W2S) & Informative & \textbf{0.071} & \textbf{96.3} & \textbf{35.0} & \textbf{73.2} & \textbf{40.0} \\
\bottomrule
\end{tabular}
}
\vspace{-10pt}
\end{table}

\section{Conclusion}

In this work, we present a study of agentic misalignment in automated workflows, characterizing it as a systematic failure of Bayesian inference over latent roles where agents' posterior collapse to generic proxy utilities. To address this, we introduced Agentic Evidence Attribution (AEA), a novel alignment framework for incorporating structured, role-specific evidence into the agent's decision process. Our empirical evaluation highlights two main conclusions. First, evidence-conditioned alignment is a powerful lever for improving multi-agent reliability, often yielding performance gains that test-time scaling cannot achieve. Second, the success of Weak-to-Strong generalization proves that small, specialized evidence models can effectively provide the orthogonal alignment signals needed to improve powerful automated workflows. This opens a promising avenue for scalable oversight in multi-agent systems, suggesting that we can build reliable automated workflows by coupling strong reasoning agents with specialized, evidence-focused aligners. We hope that this research can facilitate future strides toward trustworthy deployment of autonomous agentic systems.

\bibliographystyle{plainnat}
\bibliography{wenqian}

\begin{thebibliography}{46}
\providecommand{\natexlab}[1]{#1}
\providecommand{\url}[1]{\texttt{#1}}
\expandafter\ifx\csname urlstyle\endcsname\relax
  \providecommand{\doi}[1]{doi: #1}\else
  \providecommand{\doi}{doi: \begingroup \urlstyle{rm}\Url}\fi

\bibitem[Achiam et~al.(2023)Achiam, Adler, Agarwal, Ahmad, Akkaya, Aleman, Almeida, Altenschmidt, Altman, Anadkat, et~al.]{achiam2023gpt}
Josh Achiam, Steven Adler, Sandhini Agarwal, Lama Ahmad, Ilge Akkaya, Florencia~Leoni Aleman, Diogo Almeida, Janko Altenschmidt, Sam Altman, Shyamal Anadkat, et~al.
\newblock Gpt-4 technical report.
\newblock \emph{arXiv preprint arXiv:2303.08774}, 2023.

\bibitem[Aksitov et~al.(2024)Aksitov, Miryoosefi, Li, Li, Babayan, Kopparapu, Fisher, Guo, Prakash, Srinivasan, et~al.]{aksitov2024rest}
Renat Aksitov, Sobhan Miryoosefi, Zonglin Li, Daliang Li, Sheila Babayan, Kavya Kopparapu, Zachary Fisher, Ruiqi Guo, Sushant Prakash, Pranesh Srinivasan, et~al.
\newblock Rest meets react: Self-improvement for multi-step reasoning llm agent.
\newblock In \emph{ICLR 2024 Workshop on Large Language Model (LLM) Agents}, 2024.

\bibitem[Anthropic(2024)]{claude3}
Anthropic.
\newblock Introducing the next generation of claude, 2024.
\newblock URL \url{https://www.anthropic.com/news/claude-3-family}.

\bibitem[Bengio et~al.(2025)Bengio, Cohen, Fornasiere, Ghosn, Greiner, MacDermott, Mindermann, Oberman, Richardson, Richardson, et~al.]{bengio2025superintelligent}
Yoshua Bengio, Michael Cohen, Damiano Fornasiere, Joumana Ghosn, Pietro Greiner, Matt MacDermott, S{\"o}ren Mindermann, Adam Oberman, Jesse Richardson, Oliver Richardson, et~al.
\newblock Superintelligent agents pose catastrophic risks: Can scientist ai offer a safer path?
\newblock \emph{arXiv preprint arXiv:2502.15657}, 2025.

\bibitem[Brown et~al.(2024)Brown, Juravsky, Ehrlich, Clark, Le, R{\'e}, and Mirhoseini]{brown2024large}
Bradley Brown, Jordan Juravsky, Ryan Ehrlich, Ronald Clark, Quoc~V Le, Christopher R{\'e}, and Azalia Mirhoseini.
\newblock Large language monkeys: Scaling inference compute with repeated sampling.
\newblock \emph{arXiv preprint arXiv:2407.21787}, 2024.

\bibitem[Cemri et~al.(2025)Cemri, Pan, Yang, Agrawal, Chopra, Tiwari, Keutzer, Parameswaran, Klein, Ramchandran, et~al.]{cemri2025multi}
Mert Cemri, Melissa~Z Pan, Shuyi Yang, Lakshya~A Agrawal, Bhavya Chopra, Rishabh Tiwari, Kurt Keutzer, Aditya Parameswaran, Dan Klein, Kannan Ramchandran, et~al.
\newblock Why do multi-agent llm systems fail?
\newblock \emph{arXiv preprint arXiv:2503.13657}, 2025.

\bibitem[Chen et~al.(2021)Chen, Tworek, Jun, Yuan, Pond{\'e}, Kaplan, Edwards, Burda, Joseph, Brockman, Ray, Puri, Krueger, Petrov, Khlaaf, Sastry, Mishkin, Chan, Gray, Ryder, Pavlov, Power, Kaiser, Bavarian, Winter, Tillet, Such, Cummings, Plappert, Chantzis, Barnes, Herbert-Voss, Guss, Nichol, Babuschkin, Balaji, Jain, Carr, Leike, Achiam, Misra, Morikawa, Radford, Knight, Brundage, Murati, Mayer, Welinder, McGrew, Amodei, McCandlish, Sutskever, and Zaremba]{Chen2021EvaluatingLL}
Mark Chen, Jerry Tworek, Heewoo Jun, Qiming Yuan, Henrique Pond{\'e}, Jared Kaplan, Harrison Edwards, Yura Burda, Nicholas Joseph, Greg Brockman, Alex Ray, Raul Puri, Gretchen Krueger, Michael Petrov, Heidy Khlaaf, Girish Sastry, Pamela Mishkin, Brooke Chan, Scott Gray, Nick Ryder, Mikhail Pavlov, Alethea Power, Lukasz Kaiser, Mo~Bavarian, Clemens Winter, Phil Tillet, Felipe~Petroski Such, David~W. Cummings, Matthias Plappert, Fotios Chantzis, Elizabeth Barnes, Ariel Herbert-Voss, William~H. Guss, Alex Nichol, Igor Babuschkin, Suchir Balaji, Shantanu Jain, Andrew Carr, Jan Leike, Josh Achiam, Vedant Misra, Evan Morikawa, Alec Radford, Matthew~M. Knight, Miles Brundage, Mira Murati, Katie Mayer, Peter Welinder, Bob McGrew, Dario Amodei, Sam McCandlish, Ilya Sutskever, and Wojciech Zaremba.
\newblock Evaluating large language models trained on code.
\newblock \emph{ArXiv}, abs/2107.03374, 2021.

\bibitem[Chen et~al.(2025)Chen, Li, Wang, Jin, Qian, Wang, Wang, Zhang, Zhang, Zhang, et~al.]{chen2025rm}
Xiusi Chen, Gaotang Li, Ziqi Wang, Bowen Jin, Cheng Qian, Yu~Wang, Hongru Wang, Yu~Zhang, Denghui Zhang, Tong Zhang, et~al.
\newblock Rm-r1: Reward modeling as reasoning.
\newblock \emph{arXiv preprint arXiv:2505.02387}, 2025.

\bibitem[Erdogan et~al.(2025)Erdogan, Furuta, Kim, Lee, Moon, Anumanchipalli, Keutzer, and Gholami]{erdogan2025planandact}
Lutfi~Eren Erdogan, Hiroki Furuta, Sehoon Kim, Nicholas Lee, Suhong Moon, Gopala Anumanchipalli, Kurt Keutzer, and Amir Gholami.
\newblock Plan-and-act: Improving planning of agents for long-horizon tasks.
\newblock In \emph{Forty-second International Conference on Machine Learning}, 2025.

\bibitem[Grijalba et~al.(2024)Grijalba, Lopez, Mart{\'\i}nez-C{\'a}mara, and Camacho-Collados]{grijalba2024question}
Jorge~Os{\'e}s Grijalba, L~Alfonso~Urena Lopez, Eugenio Mart{\'\i}nez-C{\'a}mara, and Jose Camacho-Collados.
\newblock Question answering over tabular data with databench: A large-scale empirical evaluation of llms.
\newblock In \emph{Proceedings of the 2024 Joint International Conference on Computational Linguistics, Language Resources and Evaluation (LREC-COLING 2024)}, pages 13471--13488, 2024.

\bibitem[Guo et~al.(2024)Guo, Chen, Wang, Chang, Pei, Chawla, Wiest, and Zhang]{guo2024large}
Taicheng Guo, Xiuying Chen, Yaqi Wang, Ruidi Chang, Shichao Pei, Nitesh~V Chawla, Olaf Wiest, and Xiangliang Zhang.
\newblock Large language model based multi-agents: A survey of progress and challenges.
\newblock In \emph{IJCAI}, 2024.

\bibitem[Han et~al.(2024)Han, Zhang, Yao, Jin, and Xu]{han2024llm}
Shanshan Han, Qifan Zhang, Yuhang Yao, Weizhao Jin, and Zhaozhuo Xu.
\newblock Llm multi-agent systems: Challenges and open problems.
\newblock \emph{arXiv preprint arXiv:2402.03578}, 2024.

\bibitem[HuggingFace(2025{\natexlab{a}})]{huggingface_aime_2024}
HuggingFace.
\newblock Aime 2024 benchmark (huggingfaceh4/aime\_2024).
\newblock \url{https://huggingface.co/datasets/HuggingFaceH4/aime_2024}, 2025{\natexlab{a}}.

\bibitem[HuggingFace(2025{\natexlab{b}})]{huggingface_aime_2025}
HuggingFace.
\newblock Aime 2025 benchmark (opencompass/aime2025).
\newblock \url{https://huggingface.co/datasets/opencompass/AIME2025}, 2025{\natexlab{b}}.

\bibitem[Karwowski et~al.(2024)Karwowski, Hayman, Bai, Kiendlhofer, Griffin, and Skalse]{karwowskigoodhart}
Jacek Karwowski, Oliver Hayman, Xingjian Bai, Klaus Kiendlhofer, Charlie Griffin, and Joar Max~Viktor Skalse.
\newblock Goodhart's law in reinforcement learning.
\newblock In \emph{The Twelfth International Conference on Learning Representations}, 2024.

\bibitem[Kim et~al.(2025)Kim, Gu, Park, Park, Schmidgall, Heydari, Yan, Zhang, Zhuang, Malhotra, et~al.]{kim2025towards}
Yubin Kim, Ken Gu, Chanwoo Park, Chunjong Park, Samuel Schmidgall, A~Ali Heydari, Yao Yan, Zhihan Zhang, Yuchen Zhuang, Mark Malhotra, et~al.
\newblock Towards a science of scaling agent systems.
\newblock \emph{arXiv preprint arXiv:2512.08296}, 2025.

\bibitem[Korbak et~al.(2022)Korbak, Perez, and Buckley]{korbak2022rl}
Tomasz Korbak, Ethan Perez, and Christopher Buckley.
\newblock Rl with kl penalties is better viewed as bayesian inference.
\newblock In \emph{Findings of the Association for Computational Linguistics: EMNLP 2022}, pages 1083--1091, 2022.

\bibitem[Lambert et~al.(2025)Lambert, Pyatkin, Morrison, Miranda, Lin, Chandu, Dziri, Kumar, Zick, Choi, et~al.]{lambert2025rewardbench}
Nathan Lambert, Valentina Pyatkin, Jacob Morrison, Lester James~Validad Miranda, Bill~Yuchen Lin, Khyathi Chandu, Nouha Dziri, Sachin Kumar, Tom Zick, Yejin Choi, et~al.
\newblock Rewardbench: Evaluating reward models for language modeling.
\newblock In \emph{Findings of the Association for Computational Linguistics: NAACL 2025}, pages 1755--1797, 2025.

\bibitem[Li et~al.(2024)Li, Xu, Mei, Hua, Rama, Raheja, Wang, Zhu, and Zhang]{li2024autoflow}
Zelong Li, Shuyuan Xu, Kai Mei, Wenyue Hua, Balaji Rama, Om~Raheja, Hao Wang, He~Zhu, and Yongfeng Zhang.
\newblock Autoflow: Automated workflow generation for large language model agents.
\newblock \emph{arXiv preprint arXiv:2407.12821}, 2024.

\bibitem[Liu et~al.(2025)Liu, Yao, Min, Cao, Hou, and Li]{liurm2025}
Yantao Liu, Zijun Yao, Rui Min, Yixin Cao, Lei Hou, and Juanzi Li.
\newblock Rm-bench: Benchmarking reward models of language models with subtlety and style.
\newblock In \emph{The Thirteenth International Conference on Learning Representations}, 2025.

\bibitem[Lynch et~al.(2025)Lynch, Wright, Larson, Ritchie, Mindermann, Hubinger, Perez, and Troy]{lynch2025agentic}
Aengus Lynch, Benjamin Wright, Caleb Larson, Stuart~J Ritchie, Soren Mindermann, Evan Hubinger, Ethan Perez, and Kevin Troy.
\newblock Agentic misalignment: How llms could be insider threats.
\newblock \emph{arXiv preprint arXiv:2510.05179}, 2025.

\bibitem[Ma et~al.(2025)Ma, Xie, Wang, Wang, Wu, Li, and Wang]{ma2025diagnosing}
Xuyan Ma, Xiaofei Xie, Yawen Wang, Junjie Wang, Boyu Wu, Mingyang Li, and Qing Wang.
\newblock Diagnosing failure root causes in platform-orchestrated agentic systems: Dataset, taxonomy, and benchmark.
\newblock \emph{arXiv preprint arXiv:2509.23735}, 2025.

\bibitem[MacKay(2002)]{mackay}
David J.~C. MacKay.
\newblock \emph{Information Theory, Inference \& Learning Algorithms}.
\newblock Cambridge University Press, USA, 2002.
\newblock ISBN 0521642981.

\bibitem[Mialon et~al.(2023)Mialon, Fourrier, Wolf, LeCun, and Scialom]{mialon2023gaia}
Gr{\'e}goire Mialon, Cl{\'e}mentine Fourrier, Thomas Wolf, Yann LeCun, and Thomas Scialom.
\newblock Gaia: a benchmark for general ai assistants.
\newblock In \emph{The Twelfth International Conference on Learning Representations}, 2023.

\bibitem[Plaat et~al.(2025)Plaat, Wong, Verberne, Broekens, and Van~Stein]{plaat2025multi}
Aske Plaat, Annie Wong, Suzan Verberne, Joost Broekens, and Niki Van~Stein.
\newblock Multi-step reasoning with large language models, a survey.
\newblock \emph{ACM Computing Surveys}, 2025.

\bibitem[Poddar et~al.(2024)Poddar, Wan, Ivison, Gupta, and Jaques]{poddar2024personalizing}
Sriyash Poddar, Yanming Wan, Hamish Ivison, Abhishek Gupta, and Natasha Jaques.
\newblock Personalizing reinforcement learning from human feedback with variational preference learning.
\newblock In \emph{Proceedings of the 38th International Conference on Neural Information Processing Systems}, pages 52516--52544, 2024.

\bibitem[Prasad et~al.(2024)Prasad, Koller, Hartmann, Clark, Sabharwal, Bansal, and Khot]{prasad2024adapt}
Archiki Prasad, Alexander Koller, Mareike Hartmann, Peter Clark, Ashish Sabharwal, Mohit Bansal, and Tushar Khot.
\newblock Adapt: As-needed decomposition and planning with language models.
\newblock In \emph{Findings of the Association for Computational Linguistics: NAACL 2024}, pages 4226--4252, 2024.

\bibitem[Qin et~al.(2024)Qin, Liang, Ye, Zhu, Yan, Lu, Lin, Cong, Tang, Qian, et~al.]{qin2024toolllm}
Yujia Qin, Shihao Liang, Yining Ye, Kunlun Zhu, Lan Yan, Yaxi Lu, Yankai Lin, Xin Cong, Xiangru Tang, Bill Qian, et~al.
\newblock Toolllm: Facilitating large language models to master 16000+ real-world apis.
\newblock In \emph{ICLR}, 2024.

\bibitem[R{\u{a}}dulescu et~al.(2020)R{\u{a}}dulescu, Mannion, Roijers, and Now{\'e}]{ruadulescu2020multi}
Roxana R{\u{a}}dulescu, Patrick Mannion, Diederik~M Roijers, and Ann Now{\'e}.
\newblock Multi-objective multi-agent decision making: a utility-based analysis and survey.
\newblock \emph{Autonomous Agents and Multi-Agent Systems}, 34\penalty0 (1):\penalty0 10, 2020.

\bibitem[Shao et~al.(2024)Shao, Wang, Zhu, Xu, Song, Bi, Zhang, Zhang, Li, Wu, et~al.]{shao2024deepseekmath}
Zhihong Shao, Peiyi Wang, Qihao Zhu, Runxin Xu, Junxiao Song, Xiao Bi, Haowei Zhang, Mingchuan Zhang, YK~Li, Yang Wu, et~al.
\newblock Deepseekmath: Pushing the limits of mathematical reasoning in open language models.
\newblock \emph{arXiv preprint arXiv:2402.03300}, 2024.

\bibitem[Skalse et~al.(2022)Skalse, Howe, Krasheninnikov, and Krueger]{skalse2022defining}
Joar Skalse, Nikolaus Howe, Dmitrii Krasheninnikov, and David Krueger.
\newblock Defining and characterizing reward gaming.
\newblock \emph{Advances in Neural Information Processing Systems}, 35:\penalty0 9460--9471, 2022.

\bibitem[Snell et~al.(2024)Snell, Lee, Xu, and Kumar]{snell2025scaling}
Charlie Snell, Jaehoon Lee, Kelvin Xu, and Aviral Kumar.
\newblock Scaling llm test-time compute optimally can be more effective than scaling model parameters.
\newblock \emph{arXiv preprint arXiv:2408.03314}, 2024.

\bibitem[Sohl-Dickstein(2023)]{sohldickstein20230309}
Jascha Sohl-Dickstein.
\newblock { The hot mess theory of AI misalignment: More intelligent agents behave less coherently }.
\newblock \url{https://sohl-dickstein.github.io/2023/03/09/coherence.html}, 2023.

\bibitem[Spaan(2012)]{spaan2012partially}
Matthijs~TJ Spaan.
\newblock Partially observable markov decision processes.
\newblock In \emph{Reinforcement learning: State-of-the-art}, pages 387--414. Springer, 2012.

\bibitem[Wang et~al.(2025)Wang, la~Tour, Watkins, Makelov, Chi, Miserendino, Wang, Rajaram, Heidecke, Patwardhan, et~al.]{wang2025persona}
Miles Wang, Tom~Dupr{\'e} la~Tour, Olivia Watkins, Alex Makelov, Ryan~A Chi, Samuel Miserendino, Jeffrey Wang, Achyuta Rajaram, Johannes Heidecke, Tejal Patwardhan, et~al.
\newblock Persona features control emergent misalignment.
\newblock \emph{arXiv preprint arXiv:2506.19823}, 2025.

\bibitem[Wang et~al.(2024)Wang, Hu, Lu, Zhu, Zhang, Subramaniam, Loomba, Zhang, Sun, and Wang]{wang2024scibench}
Xiaoxuan Wang, Ziniu Hu, Pan Lu, Yanqiao Zhu, Jieyu Zhang, Satyen Subramaniam, Arjun~R Loomba, Shichang Zhang, Yizhou Sun, and Wei Wang.
\newblock Scibench: Evaluating college-level scientific problem-solving abilities of large language models.
\newblock In \emph{International Conference on Machine Learning}, pages 50622--50649. PMLR, 2024.

\bibitem[Weng(2024)]{weng2024rewardhack}
Lilian Weng.
\newblock Reward hacking in reinforcement learning.
\newblock \emph{lilianweng.github.io}, Nov 2024.
\newblock URL \url{https://lilianweng.github.io/posts/2024-11-28-reward-hacking/}.

\bibitem[White et~al.(2025)White, Dooley, Roberts, Pal, Feuer, Jain, Shwartz-Ziv, Jain, Saifullah, Dey, et~al.]{whitelivebench}
Colin White, Samuel Dooley, Manley Roberts, Arka Pal, Benjamin Feuer, Siddhartha Jain, Ravid Shwartz-Ziv, Neel Jain, Khalid Saifullah, Sreemanti Dey, et~al.
\newblock Livebench: A challenging, contamination-limited llm benchmark.
\newblock In \emph{The Thirteenth International Conference on Learning Representations}, 2025.

\bibitem[Wu et~al.(2024)Wu, Bansal, Zhang, Wu, Li, Zhu, Jiang, Zhang, Zhang, Liu, et~al.]{wu2024autogen}
Qingyun Wu, Gagan Bansal, Jieyu Zhang, Yiran Wu, Beibin Li, Erkang Zhu, Li~Jiang, Xiaoyun Zhang, Shaokun Zhang, Jiale Liu, et~al.
\newblock Autogen: Enabling next-gen llm applications via multi-agent conversations.
\newblock In \emph{First Conference on Language Modeling}, 2024.

\bibitem[Wu et~al.(2025)Wu, Sarthi, Zhao, Lee, Shandilya, Grobelnik, Choudhary, Huang, Subbian, Zhang, et~al.]{wu2025optimas}
Shirley Wu, Parth Sarthi, Shiyu Zhao, Aaron Lee, Herumb Shandilya, Adrian~Mladenic Grobelnik, Nurendra Choudhary, Eddie Huang, Karthik Subbian, Linjun Zhang, et~al.
\newblock Optimas: Optimizing compound ai systems with globally aligned local rewards.
\newblock \emph{arXiv preprint arXiv:2507.03041}, 2025.

\bibitem[Ye et~al.(2025)Ye, Zheng, and Zhang]{yerectifying}
Wenqian Ye, Guangtao Zheng, and Aidong Zhang.
\newblock Rectifying shortcut behaviors in preference-based reward learning.
\newblock In \emph{The Thirty-ninth Annual Conference on Neural Information Processing Systems}, 2025.

\bibitem[Yoran et~al.(2024)Yoran, Amouyal, Malaviya, Bogin, Press, and Berant]{yoran2024assistantbench}
Ori Yoran, Samuel Amouyal, Chaitanya Malaviya, Ben Bogin, Ofir Press, and Jonathan Berant.
\newblock Assistantbench: Can web agents solve realistic and time-consuming tasks?
\newblock In \emph{Proceedings of the 2024 Conference on Empirical Methods in Natural Language Processing}, pages 8938--8968, 2024.

\bibitem[Yuan et~al.(2025)Yuan, Song, Chen, Tan, Shen, Ren, Li, and Yang]{yuan2025easytool}
Siyu Yuan, Kaitao Song, Jiangjie Chen, Xu~Tan, Yongliang Shen, Kan Ren, Dongsheng Li, and Deqing Yang.
\newblock Easytool: Enhancing llm-based agents with concise tool instruction.
\newblock In \emph{Proceedings of the 2025 Conference of the Nations of the Americas Chapter of the Association for Computational Linguistics: Human Language Technologies (Volume 1: Long Papers)}, 2025.

\bibitem[Zhang et~al.(2025{\natexlab{a}})Zhang, Wang, Chen, Zhou, Wang, and Yan]{zhang2025agentracer}
Guibin Zhang, Junhao Wang, Junjie Chen, Wangchunshu Zhou, Kun Wang, and Shuicheng Yan.
\newblock Agentracer: Who is inducing failure in the llm agentic systems?
\newblock \emph{arXiv preprint arXiv:2509.03312}, 2025{\natexlab{a}}.

\bibitem[Zhang et~al.(2025{\natexlab{b}})Zhang, Xiang, Yu, Teng, Chen, Chen, Zhuge, Cheng, Hong, Wang, et~al.]{zhangaflow}
Jiayi Zhang, Jinyu Xiang, Zhaoyang Yu, Fengwei Teng, Xiong-Hui Chen, Jiaqi Chen, Mingchen Zhuge, Xin Cheng, Sirui Hong, Jinlin Wang, et~al.
\newblock Aflow: Automating agentic workflow generation.
\newblock In \emph{The Thirteenth International Conference on Learning Representations}, 2025{\natexlab{b}}.

\bibitem[Zhang et~al.(2025{\natexlab{c}})Zhang, Yin, Zhang, Liu, Han, Zhang, Li, Wang, Wang, Chen, and Wu]{zhang2025which}
Shaokun Zhang, Ming Yin, Jieyu Zhang, Jiale Liu, Zhiguang Han, Jingyang Zhang, Beibin Li, Chi Wang, Huazheng Wang, Yiran Chen, and Qingyun Wu.
\newblock Which agent causes task failures and when? on automated failure attribution of {LLM} multi-agent systems.
\newblock In \emph{Forty-second International Conference on Machine Learning}, 2025{\natexlab{c}}.

\end{thebibliography}

\clearpage
\appendix

\section*{{\Large Appendix}}



\section{Broader Impact}
\label{sec:impact}

This paper presents work whose goal is to advance the field of machine learning, specifically focusing on the safety and reliability of automated multi-agent systems (MAS). By formally characterizing \textit{agentic misalignment} and introducing Agentic Evidence Attribution (AEA) to mitigate it, our research aims to reduce the risks of uncoordinated behavior and reward hacking in increasingly autonomous workflows. As MAS are deployed in critical domains such as scientific research, software engineering, and decision support, ensuring that agents adhere to specific role-based constraints is essential for preventing unintended harmful consequences. While techniques that improve the controllability of agents could theoretically be leveraged to optimize malicious workflows, we believe that understanding the mechanisms of posterior collapse and providing tools for evidence-based alignment are fundamental prerequisites for building robust, trustworthy, and human-aligned AI systems.



\section{Limitations}
\label{sec:limitations}

AEA focuses on role misspecification in automated multi-agent workflows, which is one concrete form of agentic misalignment. Other failures, such as insufficient base-model capability, incorrect workflow decomposition, tool errors, or adversarial misuse, require different mechanisms and are outside the scope of this work. Our Bayesian formulation is intended to characterize observable coordination behavior, rather than to make claims about the internal representations of LLMs. The $\epsilon$-closeness assumption in Theorem~\ref{thm:role-collapse} is empirically supported in the homogeneous workflow setting studied in this paper (Appendix~\ref{sec:posterior_collapse_exp}), while heterogeneous workflows with different model families may induce different role-inference dynamics. AEA also adds a modest inference cost, about 6\% additional tokens in our experiments (Appendix~\ref{sec:cost_analysis}), and is expected to be most useful when errors are caused by role coordination rather than insufficient reasoning depth. Finally, the effectiveness of AEA depends on the extraction function $\mathcal{F}$ producing evidence that satisfies the Information Gain Condition. When the extracted evidence is not informative for the current workflow, AEA should be viewed as a diagnostic and corrective framework rather than a guaranteed solution to all agentic failures.

\section{Proofs of Theoretical Results}
\label{sec:proofs}

In this section, we provide full proofs for the theorems and corollaries presented in Section \ref{sec:theory} and \ref{sec:aea} of the main paper. 

\subsection{Proof of Theorem \ref{thm:role-collapse}}

\textbf{Theorem \ref{thm:role-collapse} ($\epsilon$-Stability of Role Posteriors).}
\textit{Let $i, j \in \mathcal{N}$ be two distinct roles. Assume the priors are $\epsilon_{\pi}$-close: $\| P(\cdot \mid i) - P(\cdot \mid j) \|_{\mathrm{TV}} \leq \epsilon_{\pi}$ and the likelihoods are $\epsilon_{\ell}$-close. Furthermore, assume the workflow evidence $Y$ is sufficiently informative such that the marginal likelihood is lower-bounded by $\zeta > 0$ for both roles. Then, the posterior distributions are $\delta$-close, where $\delta = \kappa (2\epsilon_{\pi} + \epsilon_{\ell})$ and $\kappa \propto \zeta^{-1}$. Consequently, the probability that agents $i$ and $j$ select distinct actions is upper-bounded:
\begin{equation}
    P(\hat{\pi}_i(Y) \neq \hat{\pi}_j(Y)) \leq \delta.
\end{equation}
If the oracle requirements differ ($\pi^*_i \neq \pi^*_j$), the probability of misalignment is lower-bounded by $1 - \delta$.}

\begin{proof}
    We derive the bound on the posterior distance. Let $P_k(\theta)$ and $L_k(\theta)$ denote the prior and likelihood for agent $k \in \{i, j\}$. The posterior is given by Bayes' rule: $P(\theta|Y, k) = L_k(\theta)P_k(\theta) / Z_k$, where $Z_k = \int L_k(\theta) dP_k(\theta)$ is the evidence.
    
    We first bound the difference in evidence $|Z_i - Z_j|$. Using the linearity of the integral and the triangular inequality:
    \begin{equation}
        |Z_i - Z_j| \leq  | \int L_i d(P_i - P_j)  | +  | \int (L_i - L_j) dP_j  |.
    \end{equation}
    Assuming likelihoods are bounded by $M$, the first term is bounded by $2M\epsilon_{\pi}$ and the second by $\epsilon_{\ell}$. Thus, the evidences are close.
    
    Next, we apply the perturbation bound for the ratio $\frac{A}{a} - \frac{B}{b} = \frac{A-B}{a} + B\frac{b-a}{ab}$. Letting $A=L_i P_i$ and $a=Z_i$, we obtain:
    \begin{equation}
        \| P(\cdot|i) - P(\cdot|j) \|_{\mathrm{TV}} \leq \frac{1}{Z_i} \| L_i P_i - L_j P_j \|_{\mathrm{TV}} + \frac{|Z_j - Z_i|}{Z_i Z_j}.
    \end{equation}
    Given the lower bound assumption $Z_k \geq \zeta$, the inverse terms are bounded by $\zeta^{-1}$. Grouping the error terms $\epsilon_{\pi}$ and $\epsilon_{\ell}$ yields a linear dependency scaled by the condition number $\kappa \approx \frac{M}{\zeta}$. Thus, strong generic priors (small $\epsilon_{\pi}$) and weak evidence (small $\zeta$) ensure posteriors remain $\delta$-close, leading to generic mode collapse. Hence, according to the definition of TV distance, we have the probability that agent $i$ and $j$  select distinct actions is also upper-bounded by $\delta$.
\end{proof}

\subsection{Proof of Theorem \ref{thm:fano}}

\textbf{Theorem \ref{thm:fano} (Lower Bound on Misalignment Error).}
\textit{Assume the action space has cardinality $|\mathcal{A}_i| = M$. Let $\hat{a}(Y)$ be any decision rule based on evidence $Y$. The probability of decisive error is lower-bounded by Fano's Inequality:
\begin{equation}
    P_e \geq 1 - \frac{I(L^*; Y) + \log 2}{\log M},
\end{equation}
where $I(L^*; Y)$ is the mutual information between the true optimal action label and the evidence.}

\begin{proof}
    We model the agent's decision process as a Markov chain $L^* \to Y \to \hat{L}$, where $L^*$ is the true optimal action and $\hat{L}$ is the agent's estimated action. Fano's Inequality relates the probability of error $P_e = P(\hat{L} \neq L^*)$ to the conditional entropy $H(L^* \mid Y)$:
    \begin{equation}
        H(L^* \mid Y) \leq H_b(P_e) + P_e \log(M - 1),
    \end{equation}
    where $H_b(P_e)$ is the binary entropy function. We apply two bounds:
    1. $H_b(P_e) \leq \log 2$ (the maximum entropy of a binary variable).
    2. $\log(M-1) < \log M$.
    Substituting these yields:
    \begin{equation}
        H(L^* \mid Y) \leq \log 2 + P_e \log M.
    \end{equation}
    By the definition of mutual information, $H(L^* \mid Y) = H(L^*) - I(L^*; Y)$. Assuming a uniform prior over the optimal actions (representing the worst-case uncertainty), $H(L^*) = \log M$. We substitute this into the inequality:
    \begin{equation}
        \log M - I(L^*; Y) \leq \log 2 + P_e \log M.
    \end{equation}
    Rearranging terms to solve for $P_e$:
    \begin{align}
        P_e \log M &\geq \log M - I(L^*; Y) - \log 2 \\
        P_e &\geq 1 - \frac{I(L^*; Y) + \log 2}{\log M}.
    \end{align}
    This completes the proof.
\end{proof}

\subsection{Proof of Corollary \ref{cor:info_gain}}

\textbf{Corollary \ref{cor:info_gain} (Information Gain Condition).}
\textit{LLet $P_e(Y)$ and $P_e(Y')$ be the probabilities of decisive error under baseline and AEA evidence, respectively. A necessary condition for AEA to strictly reduce the probabilities on misalignment (i.e., $P_e(Y') < P_e(Y)$) is
\begin{equation}
        I(L^*; E_t \mid Y_t) > 0
\end{equation}
}

\begin{proof}
    We compare the mutual information terms in the lower bound from Theorem \ref{thm:fano}. The updated evidence is $Y' = (Y, E)$. By the chain rule for mutual information:
    \begin{equation}
        I(L^*; Y, E) = I(L^*; Y) + I(L^*; E \mid Y).
    \end{equation}
    For the error bound to decrease, the subtracted term in Fano's inequality must increase:
    \begin{equation}
        \frac{I(L^*; Y') + \log 2}{\log M} > \frac{I(L^*; Y) + \log 2}{\log M}.
    \end{equation}
    This simplifies to $I(L^*; Y') > I(L^*; Y)$. Substituting the chain rule expansion, we require:
    \begin{equation}
        I(L^*; Y) + I(L^*; E \mid Y) > I(L^*; Y) \implies I(L^*; E \mid Y) > 0.
    \end{equation}
    Thus, the agentic evidence $E$ must provide strictly positive information about the optimal action $L^*$ conditioned on the existing workflow $Y$.
\end{proof}

\subsection{Proof of Theorem \ref{thm:covariance}}

\textbf{Theorem \ref{thm:covariance} (Covariance Contraction under Verifiable Evidence).}
\textit{Under the linear-Gaussian assumption, the posterior covariance $\Sigma_{Y'}$ after observing AEA evidence $E_t$ satisfies the Loewner partial order $\Sigma_{Y'} \preceq \Sigma_Y$. Specifically, the precision matrix increases by the Fisher information of the evidence:
\begin{equation}
    \Sigma_{Y'}^{-1} = \Sigma_Y^{-1} + \mathbf{H}^\top \mathbf{R}^{-1} \mathbf{H}.
\end{equation}}

\begin{proof}
    We rely on the standard Bayesian Linear Regression update equations for the precision matrix (inverse covariance). Let $\Lambda_Y = \Sigma_Y^{-1}$ be the precision of the belief given baseline evidence. The likelihood of the new evidence $E_t$ is Gaussian with covariance $\mathbf{R}$. The posterior precision $\Lambda_{Y'}$ is given by the sum of the prior precision and the observation precision (Fisher information):
    \begin{equation}
        \Sigma_{Y'}^{-1} = \Sigma_Y^{-1} + \mathbf{H}^\top \mathbf{R}^{-1} \mathbf{H}.
    \end{equation}
    To prove contraction ($\Sigma_{Y'} \preceq \Sigma_Y$), we must show that $\Sigma_Y - \Sigma_{Y'}$ is positive semi-definite (PSD). Equivalently, using the property that if $A \succeq B \succ 0$ then $B^{-1} \succeq A^{-1}$, we show $\Sigma_{Y'}^{-1} \succeq \Sigma_Y^{-1}$.
    Consider the term $\mathbf{Q} = \mathbf{H}^\top \mathbf{R}^{-1} \mathbf{H}$. Since $\mathbf{R}$ is a valid covariance matrix, it is symmetric positive definite ($\mathbf{R} \succ 0$), implying $\mathbf{R}^{-1} \succ 0$. For any non-zero vector $\mathbf{v} \in \mathbb{R}^d$:
    \begin{equation}
        \mathbf{v}^\top \mathbf{Q} \mathbf{v} = \mathbf{v}^\top \mathbf{H}^\top \mathbf{R}^{-1} \mathbf{H} \mathbf{v} = (\mathbf{H}\mathbf{v})^\top \mathbf{R}^{-1} (\mathbf{H}\mathbf{v}).
    \end{equation}
    Let $\mathbf{z} = \mathbf{H}\mathbf{v}$. Then $\mathbf{z}^\top \mathbf{R}^{-1} \mathbf{z} \geq 0$ because $\mathbf{R}^{-1}$ is positive definite. Thus, $\mathbf{Q} \succeq 0$.
    Since $\Sigma_{Y'}^{-1} = \Sigma_Y^{-1} + \mathbf{Q}$ and $\mathbf{Q} \succeq 0$, it follows that $\Sigma_{Y'}^{-1} \succeq \Sigma_Y^{-1}$. By inversion monotonicity, $\Sigma_{Y'} \preceq \Sigma_Y$.
\end{proof}

\subsection{Proof of Corollary \ref{cor:error_reduction}}

\textbf{Corollary \ref{cor:error_reduction} (Reliability Improvement via Variance Reduction).}
\textit{Consider two actions $a_1, a_2$ with feature difference $\delta = \phi(s,a_1) - \phi(s,a_2)$ under state $s$. If $a_1$ is optimal ($\mathbf{w}^\top \delta > 0$) and the evidence $E_t$ is unbiased (does not decrease the margin mean $\mu$), then $P_e(Y') \leq P_e(Y)$.}

\begin{proof}
    Let the random variable $D = \mathbf{w}^\top \delta$ represent the projected utility difference. Under the posterior $Y$, $D$ follows a univariate Gaussian distribution:
    \begin{equation}
        D \mid Y \sim \mathcal{N}(\mu_D, \sigma_D^2),   \text{where } \mu_D = \boldsymbol{\mu}_Y^\top \delta,   \sigma_D^2 = \delta^\top \Sigma_Y \delta.
    \end{equation}
    The agent commits an error if it samples a utility vector $\mathbf{w}$ such that $a_2$ appears better than $a_1$, i.e., $D < 0$. The probability of this event is given by the Cumulative Distribution Function (CDF) of the standard normal distribution, $\Phi$:
    \begin{equation}
        P_e(Y) = P(D < 0 \mid Y) = \Phi (-\frac{\mu_D}{\sigma_D} ).
    \end{equation}
    From Theorem \ref{thm:covariance}, we know $\Sigma_{Y'} \preceq \Sigma_Y$. By the definition of PSD matrices, this implies variance reduction along any direction $\delta$:
    \begin{equation}
        \sigma_{D'}^2 = \delta^\top \Sigma_{Y'} \delta \leq \delta^\top \Sigma_Y \delta = \sigma_D^2.
    \end{equation}
    Let $\sigma_{D'} \leq \sigma_D$. We assume unbiased evidence such that the posterior mean margin does not degrade: $\mu_{D'} \geq \mu_D > 0$.
    We analyze the argument of the CDF $\Phi(x)$. Since $\mu_D > 0$ and $\sigma_D > 0$:
    \begin{equation}
        \frac{\mu_{D'}}{\sigma_{D'}} \geq \frac{\mu_D}{\sigma_D} \implies -\frac{\mu_{D'}}{\sigma_{D'}} \leq -\frac{\mu_D}{\sigma_D}.
    \end{equation}
    Since $\Phi(z)$ is a strictly monotonically increasing function, a smaller (more negative) argument yields a smaller probability:
    \begin{equation}
        \Phi (-\frac{\mu_{D'}}{\sigma_{D'}} ) \leq \Phi (-\frac{\mu_D}{\sigma_D} ).
    \end{equation}
    Thus, $P_e(Y') \leq P_e(Y)$.
\end{proof}

\section{Quantitative Analyses of Posterior Collapse and AEA}
\label{sec:additional_qual}

The Bayesian framework in Section~\ref{sec:theory} produces several behavioral predictions that go beyond final-task accuracy. In this section, we report a series of controlled analyses that probe these predictions on the same multi-agent setup as Section~\ref{sec:imp_details}: how different roles converge to a shared policy, whether the evidence injected by AEA carries new information about the optimal action, how the posterior contracts under AEA, how AEA scales with workflow complexity, and how AEA compares against pure test-time scaling on the compute axis. These analyses characterize \textit{how} AEA changes agent behavior, not only \textit{whether} it improves accuracy.

\subsection{Posterior Collapse}
\label{sec:posterior_collapse_exp}

\noindent \textbf{Setup.}
We test whether unaligned agents actually behave as if they share a single policy across roles, the prediction made by Theorem~\ref{thm:role-collapse}. To make this prediction directly testable, we instantiate a three-role workflow on \texttt{AIME24} with maximally distinct roles: a \textit{Problem Solver} that solves the problem from scratch, a \textit{Solution Verifier} that checks the proposed solution \emph{without re-solving it}, and a \textit{Result Reporter} that only formats the final answer. We measure two behavioral quantities, both judged by GPT-4o with fixed rubrics. \textit{Pairwise Functional Overlap} measures, for every ordered pair of agents, how often they perform functionally equivalent operations across 90 comparisons; this serves as a behavioral proxy for $P(\hat{\pi}_i = \hat{\pi}_j)$. \textit{Role Action Accuracy} measures the fraction of turns whose action matches the assigned role description; a 50-sample human spot-check shows 96\% agreement with the judge.

\begin{table}[!t]
\centering
\caption{Behavioral signature of posterior collapse on \texttt{AIME24} with three maximally distinct roles. Without alignment, agents converge to a shared generic policy with high functional overlap and low role action accuracy, empirically confirming Theorem~\ref{thm:role-collapse}. AEA breaks this convergence, and Weak-to-Strong evidence is the most effective.}
\label{tab:posterior_collapse}
\begin{tabular}{lcc}
\toprule
\textbf{Setting} & \textbf{Functional Overlap (\%) $\downarrow$} & \textbf{Role Action Accuracy (\%) $\uparrow$} \\
\midrule
Multi-Agent (no alignment) & 58.3 & 34.2 \\
Self-Reflection (AEA) & 47.1 & 42.8 \\
Weak-to-Strong (AEA) & \textbf{29.5} & \textbf{59.6} \\
\bottomrule
\end{tabular}
\end{table}

\noindent \textbf{Unaligned agents collapse onto a shared policy.}
As shown in Table~\ref{tab:posterior_collapse}, even with maximally distinct role specifications, unaligned agents exhibit 58.3\% functional overlap while only 34.2\% of turns are role-aligned. This is exactly the behavior predicted by Theorem~\ref{thm:role-collapse}: under $\epsilon$-close priors and likelihoods, the dominant pre-training prior pulls all agents toward the same generic policy, regardless of their nominal role. Self-Reflection only mildly attenuates the collapse (47.1\% overlap, 42.8\% role accuracy), because the reflecting agent inherits the same prior that produced the collapse in the first place. In contrast, Weak-to-Strong AEA halves the overlap to 29.5\% and nearly doubles role action accuracy to 59.6\%. This empirically confirms that posterior collapse is the operative failure mode, and that breaking it requires a prior-distinct evidence channel rather than additional same-prior reasoning.

\subsection{Posterior Contraction with AEA}
\label{sec:covariance_exp}

\noindent \textbf{Setup.}
Theorem~\ref{thm:covariance} predicts that AEA tightens the posterior over the agent's latent utility, $\Sigma_{Y'} \preceq \Sigma_Y$. As an observable surrogate for this contraction, we re-run each of 30 randomly sampled tasks $k=5$ times under each setting, embed every final answer with a fixed sentence encoder, and measure the per-task variance of the embeddings around their per-task mean.

\begin{table}[!t]
\centering
\caption{Per-task variance of final-answer embeddings across $30$ tasks $\times$ $5$ repetitions. AEA roughly halves the variance, the behavioral analogue of the closed-form variance reduction $\Sigma_{Y'}^{-1} = \Sigma_Y^{-1} + \mathbf{H}^\top \mathbf{R}^{-1} \mathbf{H}$ in Theorem~\ref{thm:covariance}.}
\label{tab:covariance_contraction}
\begin{tabular}{lc}
\toprule
\textbf{Setting} & \textbf{Mean Output Embedding Variance} \\
\midrule
Multi-Agent (no alignment) & 0.0847 \\
Self-Reflection (AEA) & 0.0691 \\
Weak-to-Strong (AEA) & \textbf{0.0423} \\
\bottomrule
\end{tabular}
\end{table}

\noindent \textbf{Variance reduction matches the closed-form prediction.}
As shown in Table~\ref{tab:covariance_contraction}, the unaligned workflow yields a per-task variance of 0.0847, indicating that repeated runs of the same task produce systematically different answers, the empirical signature of a broad posterior. Self-Reflection contracts the variance modestly to 0.0691, since same-prior reflection only weakly constrains the posterior. Weak-to-Strong AEA roughly halves the unaligned variance, reaching 0.0423. This is the behavioral experiment of the closed-form variance reduction $\Sigma_{Y'}^{-1} = \Sigma_Y^{-1} + \mathbf{H}^\top \mathbf{R}^{-1} \mathbf{H}$ in Theorem~\ref{thm:covariance}: external verifiable evidence projects the latent utility onto observable constraints, sharpens the posterior, and produces more reproducible answers across runs.

\subsection{AEA Outperforms Test-Time Scaling at Lower Cost}
\label{sec:cost_analysis}

\noindent \textbf{Setup.}
A natural alternative to AEA is to spend the same additional compute on test-time scaling, such as majority voting or best-of-$K$ sampling~\citep{brown2024large,snell2025scaling}. To compare them on a common compute axis, we evaluate AEA against single-agent test-time scaling baselines on \texttt{AIME24}, \texttt{AIME25}, and \texttt{DataBench} using GPT-4o, and report the average total token consumption.

\begin{table}[!t]
\centering
\caption{Cost-normalized comparison on GPT-4o. AEA adds only $\sim$6\% token overhead, yet outperforms a strictly larger Best-of-$K$ budget on two of three benchmarks and matches it on the third.}
\label{tab:cost_analysis}
\begin{tabular}{lc ccc}
\toprule
\textbf{Method} & \textbf{Total Tokens (avg)} & \texttt{AIME24} & \texttt{AIME25} & \texttt{DataBench} \\
\midrule
Single Agent ($1\times$) & $\sim$\,950 & 16.7 & 10.0 & 27.3 \\
Majority Vote ($k{=}5$) & $\sim$\,4{,}750 & 26.7 & 16.7 & 31.9 \\
Best-of-$K$ ($k{=}15$) & $\sim$\,14{,}250 & 30.0 & 20.0 & 33.0 \\
Multi-Agent (no AEA) & $\sim$\,11{,}400 & 23.3 & 33.3 & 30.7 \\
Multi-Agent + AEA & $\sim$\,12{,}100 & \textbf{40.0} & \textbf{33.3} & \textbf{35.0} \\
\bottomrule
\end{tabular}
\end{table}

\noindent \textbf{Compute cannot substitute for evidence.}
As shown in Table~\ref{tab:cost_analysis}, AEA adds only $\sim$6\% token overhead over the unaligned multi-agent workflow ($\sim$12,100 vs.\ $\sim$11,400). Even when test-time scaling is given a strictly larger compute budget (Best-of-$K$ at $\sim$14,250 tokens), AEA still outperforms it on \texttt{AIME24} (40.0 vs.\ 30.0) and \texttt{DataBench} (35.0 vs.\ 33.0), and matches it on \texttt{AIME25}. The gain from AEA cannot be replicated by raw test-time compute. The bottleneck in these workflows is role coordination rather than reasoning depth, and pumping more samples through the same pre-training prior simply produces more confident misaligned answers, not aligned ones. This empirically supports the central claim that evidence-conditioned alignment is a more efficient lever than naive scaling.

\subsection{AEA Gains Scale with Workflow Complexity}
\label{sec:complexity_ablation}

\noindent \textbf{Setup.}
The Discussion in Appendix~\ref{sec:takeaways} argues that each additional agent enlarges the misalignment surface area, since each new agent introduces another decision node where the posterior may collapse. We test this prediction by grouping all GPT-4o evaluations from Table~\ref{tab:main} by the number of agents the orchestrator (CaptainAgent) instantiated for each task, and reporting the unaligned and AEA-aligned accuracy in each bucket.

\begin{table}[!t]
\centering
\caption{Effect of workflow complexity on GPT-4o. The unaligned baseline degrades as the number of agents grows, while the AEA gain rises monotonically. The largest gains appear in the most complex workflows.}
\label{tab:complexity_ablation}
\begin{tabular}{lcccc}
\toprule
\textbf{Agent Count} & \textbf{Tasks} & \textbf{Multi-Agent Acc.} & \textbf{AEA Acc.} & $\Delta$ \\
\midrule
2--3 agents & 415 & 52.3 & 57.8 & +5.5 \\
4--5 agents & 135 & 41.5 & 52.6 & +11.1 \\
6+ agents & 6 & 33.3 & 50.0 & +16.7 \\
\bottomrule
\end{tabular}
\end{table}

\noindent \textbf{The misalignment surface area widens with workflow size.}
As shown in Table~\ref{tab:complexity_ablation}, two trends emerge together. First, the unaligned multi-agent accuracy drops monotonically with agent count (52.3 $\to$ 41.5 $\to$ 33.3), confirming that scaling the workflow without alignment is actively harmful. Second, the AEA gain $\Delta$ rises monotonically ($+5.5 \to +11.1 \to +16.7$). The benefit of AEA therefore concentrates exactly where the unaligned baseline degrades the most. This is the empirical realization of the misalignment-surface-area argument: larger workflows have more decision nodes that can suffer from posterior collapse, and the value of an evidence channel grows accordingly.


\section{Standard Reward Evaluations}
\label{sec:reward_eval}
Our AEA evidence model is trained to produce role-specific, structured feedback for multi-agent traces. Since this training objective differs from standard reward modeling, we additionally evaluate whether the resulting model still behaves like a strong general-purpose reward reasoning model on widely used reward benchmarks. Concretely, we report results on RewardBench~\citep{lambert2025rewardbench} and RM-Bench~\citep{liurm2025}, which test preference judgment quality across chat helpfulness, safety, and reasoning, as well as math and code domains. \textit{Note that this evaluation is not our main target but only to verify the evidence model aligns with similar preference to human.}

Tables~\ref{tab:rewardbench} and~\ref{tab:rmbench} show that AEA-4B (trained with supervised warm-start followed by GRPO on our agentic trace data) maintains competitive performance compared to its underlying Qwen3 base model and to RM-R1 models \citep{chen2025rm} trained directly for reward reasoning. In particular, AEA-4B remains strong on the safety and reasoning categories, indicating that specializing the model for agentic evidence attribution does not collapse its standard preference evaluation ability. Overall, these results support the view that AEA can be used as a plug-in evidence model for workflow alignment without sacrificing the reward evaluation capability that is typically expected from reward reasoning models.

\begin{table*}[!t]
\centering
\caption{Reward Evaluation on Reward Reasoning Model (RewardBench)}
\label{tab:rewardbench}
\resizebox{0.85\textwidth}{!}{
\begin{tabular}{l l c | c c c c | c}
\toprule
Model & Training & Data & Chat & Chat Hard & Safety & Reasoning & Overall \\
\midrule
RM-R1(Qwen2.5-7B) & GRPO & RM-R1 & 94.69 & 74.34 & 85.00 & 86.11 & 85.03 \\
RM-R1(Qwen2.5-14B) & GRPO & RM-R1 & 93.57 & 79.38 & 87.70 & 91.40 & 88.01 \\
RM-R1(Qwen2.5-32B) & GRPO & RM-R1 & 93.29 & 82.02 & 90.95 & 94.72 & 90.25 \\
RM-R1(Qwen3-0.6B) & GRPO & RM-R1 & 80.72 & 60.52 & 63.37 & 49.64 & 63.56 \\
RM-R1(Qwen3-4B) & GRPO & RM-R1 & 91.62 & 75.87 & 88.11 & 86.52 & 86.13 \\ 
\bottomrule
\end{tabular}
}
\end{table*}

\begin{table*}[!t]
\centering
\caption{Reward Evaluation on Reward Reasoning Model (RM-Bench)}
\label{tab:rmbench}
\resizebox{\textwidth}{!}{
\begin{tabular}{l l c | c c c c c c c | c}
\toprule
Model & Training & Data & Chat & Math & Code & Safety & Easy & Normal & Hard & Avg \\
\midrule
RM-R1(Qwen2.5-7B) & GRPO  & RM-R1 & 67.26 & 69.41 & 55.26 & 92.18 & 80.96 & 71.87 & 60.26 & 71.03 \\
RM-R1(Qwen2.5-14B) & GRPO & RM-R1 & 75.10 & 76.68 & 60.18 & 85.46 & 81.17 & 75.23 & 66.67 & 74.36 \\
RM-R1(Qwen2.5-32B) & GRPO & RM-R1 & 75.27 & 81.39 & 64.81 & 93.67 & 85.69 & 81.01 & 69.65 & 78.79 \\
RM-R1(Qwen3-0.6B) & GRPO & RM-R1 & 55.04 & 61.73 & 53.65 & 63.24 & 73.75 & 58.36 & 43.14 & 58.42 \\
RM-R1(Qwen3-4B) & GRPO & RM-R1 & 70.63 & 72.19 & 57.21 & 92.28 & 81.44 & 73.99 & 63.75 & 73.06 \\ 
\bottomrule
\end{tabular}
}
\end{table*}


\section{Discussion}
\label{sec:takeaways}

\noindent \textbf{The Efficiency-Reliability Trade-off.}
Current trends in agentic AI emphasize scaling by increasing the number of agents and interaction depth~\citep{kim2025towards}. However, our results highlight a critical limitation. Scaling without alignment introduces a ``curse of dimensionality'' in reliability. Each additional agent introduces a new decision node where the generic posterior may fail. This effectively expands the misalignment surface area, defined as the cumulative probability that a local role deviation propagates into a global failure. AEA fundamentally alters this trade-off. It converts minor additional compute into posterior variance reduction. We find that paying this inference cost is far more efficient than naive scaling. Future architectures should prioritize the \textit{density of evidence}, ensuring $I(L^*; E \mid Y) > 0$.

\noindent \textbf{Capability vs. Evidence.}
Our Bayesian framework provides a rigorous diagnostic tool for distinguishing between two fundamental failure modes. When AEA intervention leads to success (a \textbf{Fixed} outcome), we can retrospectively attribute the initial failure to \textit{missing evidence} where the agent had the capability to solve the task but inferred the wrong utility function due to weak signals. Conversely, failures that persist even after AEA intervention likely represent \textit{missing capability}, indicating fundamental deficits in reasoning or domain knowledge. This separation clarifies whether resources should be allocated to pre-training stronger base models (to fix capability) or to engineering better runtime context and attribution signals (to fix evidence).

\noindent \textbf{Relation to Concurrent Works.}
With the rise of LLM-based agents, automated workflow generation has become a common way to build multi-agent systems (MAS). In parallel to our work, recent papers have started to study two directions that are necessary for making these systems reliable: (i) measuring and localizing failures inside a workflow, and (ii) understanding when scaling the number of agents and interaction steps actually helps. First, \citet{zhang2025which} introduce \texttt{Who\&When}, a benchmark that targets \textit{automatic failure attribution} in LLM multi-agent systems by asking models to identify which agent caused the failure and when it happened. Related efforts move from benchmarking to data generation: \citet{zhang2025agentracer} propose an automated pipeline for analyzing failed trajectories and producing structured annotations that can be used for debugging or training attribution models. In a complementary direction, \citet{kim2025towards} study the scaling behavior of agent systems and highlight that adding agents or steps is not a free win. Improvements depend on how coordination is organized and where errors accumulate in the workflow. Beyond these, there are also concurrent diagnostics~\citep{ma2025diagnosing} that analyze root causes in orchestrated agent platforms, which further supports the need for explicit failure localization rather than only scaling compute. Our goal is different from attribution or scaling alone: we argue that both should ultimately serve \textit{agentic alignment}. Concretely, we provide a bridge between (a) identifying decisive errors inside a trajectory and (b) correcting future decisions by changing the evidence available to the agent. This is why our framework treats attribution outputs as \textit{evidence} that contracts an agent's latent utility posterior, rather than as a standalone diagnosis. Empirically, this lets us turn failure attribution into a practical alignment mechanism (AEA), and it also explains why naive self-reflection can fail under a shared prior while weak-to-strong evidence can remain discriminative.

\begin{table*}[!t]
\centering
\small
\caption{Datasets used to generate agentic reasoning traces for training Agentic Evidence Attribution (AEA). The first three benchmarks (\texttt{GAIA}, \texttt{AssistantBench}, \texttt{LiveBench}) are single-agent evaluations that we convert into multi-agent runs, while \texttt{Who\&When} is a native multi-agent failure attribution benchmark. We then construct \textbf{Ours}, a larger multi-agent trace dataset with misalignment diagnostics and human-verified annotations.}
\label{tab:datasets_construct}
\resizebox{\textwidth}{!}{
\begin{tabular}{l c p{0.5\textwidth} c}
\toprule
Benchmark & Multi-Agent & Description & Samples \\
\midrule
\texttt{GAIA} \citep{mialon2023gaia} & \xmark  & Benchmark agent's reasoning, tool use, web browsing, and multi-modality in real-world tasks. & 165  \\ \midrule
\texttt{AssistantBench} \citep{yoran2024assistantbench} & \xmark & Benchmark web agents on realistic multi-step tasks across diverse domains. &  33 \\ \midrule
\texttt{LiveBench} \citep{whitelivebench} & \xmark & Benchmark with tasks in math, coding, reasoning, language, instruction following, and data analysis. & 1436 \\ \midrule
\texttt{Who\&When} \citep{zhang2025which} & \cmark & Benchmarking automatic failure attribution in multi‑agent LLM systems. & 184  \\ \midrule
\textbf{Ours} & \cmark & Multi-agent Misalignment Traces and Diagnostics & \textbf{1531} \\
\bottomrule
\end{tabular}
}
\end{table*}

\section{Prompt Engineering in Evaluation}
\label{sec:prompts}

The following prompts constitute the core interface for our evaluation pipeline. They are designed to strictly enforce structured outputs (JSON) to facilitate automated parsing and downstream feedback injection.

\subsection{AEA System Prompt}

This prompt serves as the inference interface for our fine-tuned evidence model (AEA-4B). It is designed to take the raw, multi-turn conversation history of a workflow and extract a structured failure attribution. Unlike standard chat prompts, this system instruction enforces a strict JSON schema that maps directly to the ``decisive error'' definition: identifying the specific agent role and step number where the trajectory diverged from optimality.

\begin{tcolorbox}[title=\textbf{AEA System Prompt}, colback=gray!10!white, colframe=black]
\small
You are an AI assistant that analyzes multi-agent conversations.

Your task is to analyze the conversation and identify issues with agent collaboration.

You MUST respond with a valid JSON object in this exact format:

\texttt{\{"rating": 1-5, \}}\\
\texttt{\{"agent\_name": "<agent>",\}}\\
\texttt{\{"step\_number": <int>,\}}\\
\texttt{\{"reason": "<text>",\}}\\
\texttt{\{"revised\_prompt": "<new prompt>"\}}

\textbf{Analysis guidelines:}
Rate collaboration 1-5 (1=poor, 5=excellent). Identify ONE agent that made the most critical mistake and the step number where it occurred. Note: Computer\_terminal is not an agent. Provide a brief reason and improved system prompt.
\end{tcolorbox}

\subsection{Self-Reflection Prompt}

To evaluate the ``Dominant Prior'' hypothesis, we utilize a Self-Reflection prompt that queries the backbone model (e.g., GPT-4o, Claude) to critique its own generated trace. This prompt mirrors the schema of the AEA prompt to ensure a fair comparison. Crucially, the model is provided with the full task context—including the problem statement, ground truth (for experimental analysis), and the agent's final answer—and acts as a reviewer to detect the first instance of critical failure.

\begin{tcolorbox}[title=\textbf{Self-Reflection Prompt Template}, colback=gray!10!white, colframe=black]
\small
You are analyzing a multi-agent conversation to identify collaboration issues and provide improvement suggestions.

\textbf{Task Context}\\
Problem: \{problem\}\\
Expected Answer: \{ground\_truth\}\\
Agent's Final Answer: \{agent\_answer\}

\textbf{Multi-Agent Conversation}\\
\{formatted\_history\}

\textbf{Your Analysis Task}\\
Analyze this multi-agent conversation and identify:

$\bullet$ Rate the overall collaboration quality (1-5) \\
$\bullet$ Identify the ONE agent that made the most critical mistake \\
$\bullet$ Find the step number where the mistake first occurred \\
$\bullet$ Provide a brief reason explaining the issue \\
$\bullet$ Suggest an improved system prompt for that agent \\

Note: ``computer\_terminal'' is a tool, not an agent. Focus on actual agents.

\textbf{Required Output Format}\\
Respond with a valid JSON object containing: rating, agent\_name, step\_number, reason, revised\_prompt
\end{tcolorbox}

\subsection{LLM-as-a-Judge Prompt}

To standardize performance metrics across our diverse benchmark suite, we employ a strong generalist model (GPT-4o) as an objective judge. This prompt converts the potentially verbose or unstructured outputs of the multi-agent system into a binary correctness label. The prompt is conditioned on the specific domain logic (e.g., checking mathematical derivations for \texttt{AIME} vs. code execution signatures for \texttt{HumanEval}) to ensure high-fidelity evaluation.

\begin{tcolorbox}[title=\textbf{LLM Judge Prompt Template}, colback=gray!10!white, colframe=black]
\small
You are evaluating a \{solution\_type\} from a multi-agent system. Please determine if the final answer is correct.

\textbf{Problem:} \{question\}

\textbf{Expected Answer:} \{expected\_answer\}

\textbf{Multi-Agent Response:} \{model\_response\}

Please evaluate if the multi-agent response correctly solves the problem. Look for:

$\bullet$ \textbf{\texttt{AIME}:} Correct reasoning, final numerical answer, valid approach\\
$\bullet$ \textbf{\texttt{HumanEval}:} Correct signature, implementation logic, edge cases\\
$\bullet$ \textbf{\texttt{DataBench}:} Correct answer, data analysis approach, reasoning\\
$\bullet$ \textbf{\texttt{SciBench}:} Correct scientific reasoning, calculations, principles

Respond with exactly one of: \textbf{CORRECT} (solution is right) or \textbf{INCORRECT} (solution is wrong)

\end{tcolorbox}

\begin{figure}[!t]
    \centering
    \begin{subfigure}[b]{\linewidth}
        \centering
        \includegraphics[width=\linewidth]{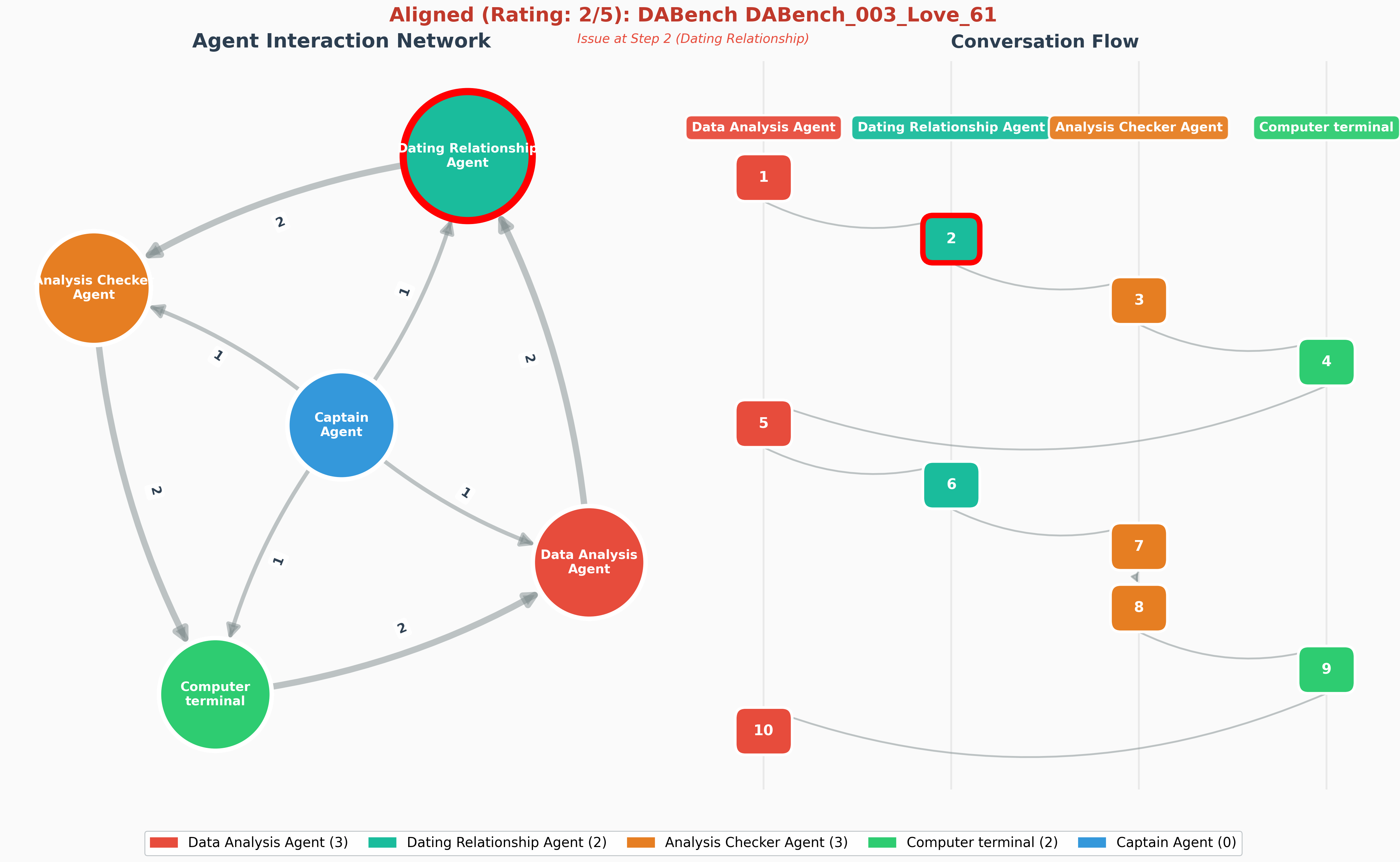}
        \caption{\textbf{AEA on Claude 3 Haiku (DataBench):} AEA successfully flags a decisive error at Step 2 (Red), assigning a low rating (2/5) and preventing misalignment propagation. \textcolor{red}{Red boarder} denotes the groundtruth misaligned step.}
        \label{fig:topo_aea}
    \end{subfigure}
    
    \vspace{0.3cm} 
    
    \begin{subfigure}[b]{\linewidth}
        \centering
        \includegraphics[width=\linewidth]{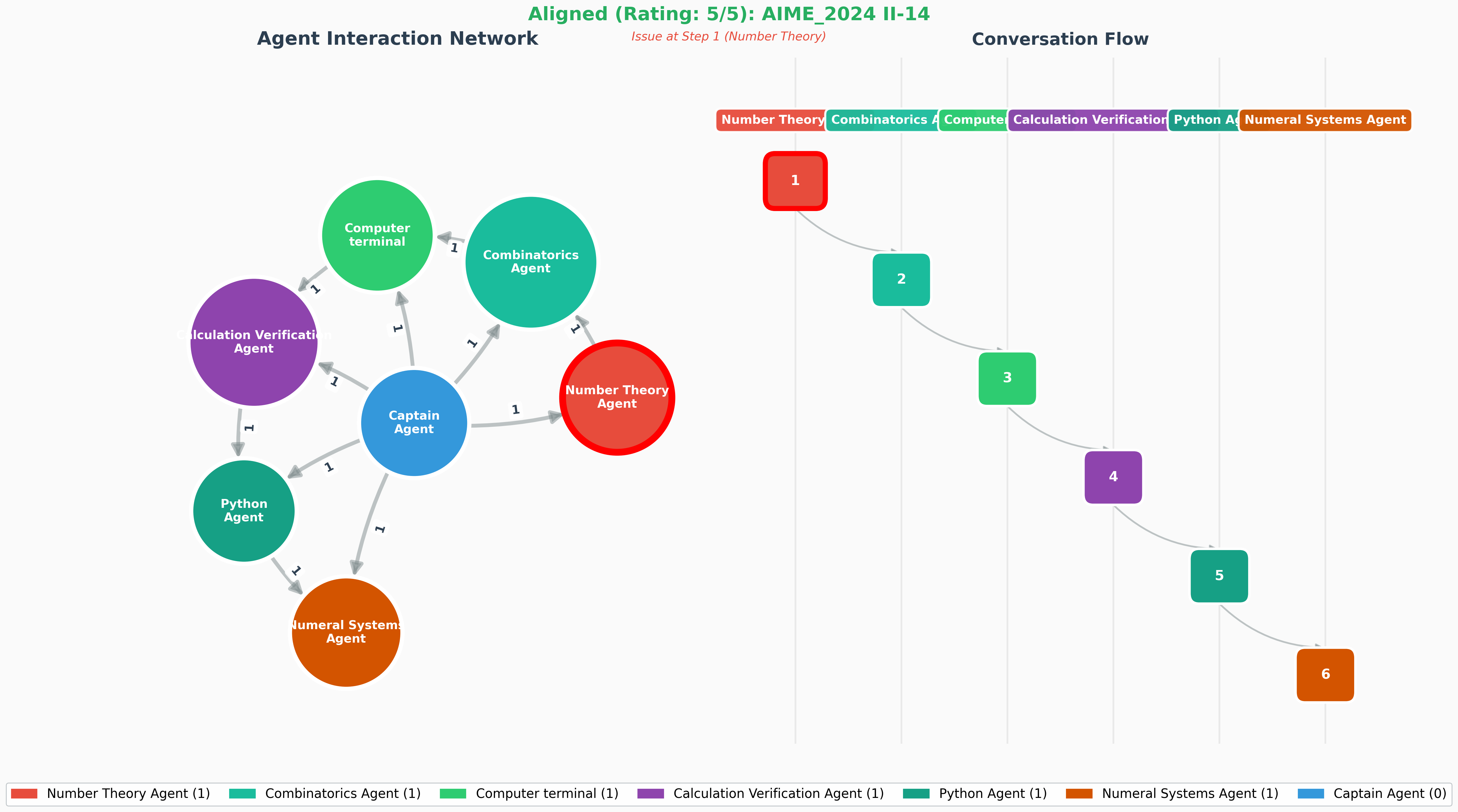}
        \caption{\textbf{Self-Reflection on GPT-4o (AIME):} The system exhibits \textit{Hallucinated Compliance}. Despite an issue at Step 1, the shared prior leads to a perfect rating (5/5), failing to catch the error. \textcolor{red}{Red boarder} denotes the groundtruth misaligned step.}
        \label{fig:topo_selfref}
    \end{subfigure}
    \caption{Topological comparison of failure modes. \textbf{(a)} AEA acts as a variance reducer, providing orthogonal evidence to identify latent failures. \textbf{(b)} Self-Reflection suffers from posterior collapse, validating incorrect trajectories due to the Dominant Prior.}
    \label{fig:topo_analysis}
\end{figure}

\section{Visualization on Topological Graphs}
\label{sec:topo_analysis}

To scrutinize the structural dynamics of misalignment, we visualize the workflow execution topology in Figure~\ref{fig:topo_analysis}. These graphs map the agent interaction network against the temporal conversation flow, allowing us to trace the propagation of decisive errors.

\noindent \textbf{AEA enables precise attribution.}
Figure~\ref{fig:topo_analysis} (Top) displays a \texttt{DataBench} workflow using Claude 3 Haiku aligned with AEA. The system encounters a decisive error at Step 2 by the \textit{Dating Relationship Agent}. In a standard workflow, this error propagates past the \textit{Analysis Checker} (Step 3). However, AEA correctly identifies the latent failure, assigning a low alignment rating ($2/5$) and pinpointing the specific node. This empirically demonstrates that AEA injects the necessary \textit{orthogonal evidence} (Corollary~\ref{cor:info_gain}) to break the error cascade.

\noindent \textbf{Self-Reflection induces posterior collapse.}
Conversely, Figure~\ref{fig:topo_analysis} (Bottom) illustrates a failure mode on \texttt{AIME} using GPT-4o. Despite an underlying issue at Step 1, the Self-Reflection mechanism assigns a perfect rating ($5/5$). This visualizes the \textbf{Dominant Prior} phenomenon (Theorem~\ref{thm:role-collapse}): because the checking mechanism shares the same prior as the acting agent, it cannot distinguish the error from valid reasoning. The system consequently ``rationalizes'' the mistake, collapsing into a misaligned state of \textit{hallucinated compliance}.

\clearpage

\section{Case Studies}
\label{sec:case}
We provide the case studies of the two instantiation of AEA. They are primarily colored in three types: (1) Fixed Error means the MAS answer turns from incorrect to correct, (2) Stayed Wrong means the MAS answer keep give the incorrect answer, and (3) Regression means the MAS turns correct answer to incorrect.

\subsection{Weak-to-Strong Cases}

\begin{tcolorbox}[title=\textbf{AEA}: Fixed Error (\texttt{AIME24}), colback=green!10!white, colframe=green!50!black]
\small
\textbf{Transition:} incorrect $\rightarrow$ correct

\textbf{Question:} Let $\mathcal{B}$ be the set of rectangular boxes with surface area $54$ and volume $23$. Let $r$ be the radius of the smallest sphere that can contain each of the rectangular boxes that are elements of $\mathcal{B}$. The value of $r^2$\ldots

\textbf{Original Answer:} The final answer is 45 (which is p + q, where $r^2 = 41/4$ with p = 41 and q = 4). \xmark

\textbf{Agent Flagged:} \texttt{Geometry\_Expert} at Step 4

\textbf{Alignment Feedback:} \textit{In step 4, the Geometry\_Expert incorrectly states that the dimensions $2 \times 2 \times 5.75$ give an exact value of $r^2 = 41/4$. However, the exact value is $r^2 = 41.0625/4$, which simplifies to $41/4$ only if $5.75$ is considered an exact fraction.}

\textbf{After Alignment:} Based on the expert conversation, the final answer is: $p + q = 721$. This is because $r^2 = 657/64$, where $p = 657$ and $q = 64$ are relatively prime. \cmark
\end{tcolorbox}

\begin{tcolorbox}[title=\textbf{AEA}: Fixed Error (\texttt{AIME25}), colback=green!10!white, colframe=green!50!black]
\small
\textbf{Transition:} incorrect $\rightarrow$ correct

\textbf{Question:} Six points $A, B, C, D, E$ and $F$ lie in a straight line in that order. Suppose that $G$ is a point not on the line and that $AC = 26$, $BD = 22$, $CE = 31$, $DF = 33$, $AF = 73$, $CG = 40$, and $DG = 30$. Find the area of $\triangle BGE$\ldots

\textbf{Original Answer:} The area of triangle $\triangle BGE$ is 468 square units. \xmark

\textbf{Agent Flagged:} \texttt{TriangleArea\_Expert} at Step 2

\textbf{Alignment Feedback:} \textit{The TriangleArea\_Expert provided a solution that directly calculated the area using coordinates and the formula, but failed to verify whether the coordinates were consistent with the given distances $CG = 40$ and $DG = 30$.}

\textbf{After Alignment:} The conversation involved calculating the area of triangle $BGE$ based on the coordinates of six points $A, B, C, D, E, F$ arranged in a straight line\ldots \cmark
\end{tcolorbox}

\begin{tcolorbox}[title=\textbf{AEA}: Fixed Error (\texttt{HumanEval}), colback=green!10!white, colframe=green!50!black]
\small
\textbf{Transition:} incorrect $\rightarrow$ correct

\textbf{Question:} \texttt{def odd\_count(lst):} Given a list of strings, where each string consists of only digits, return a list. Each element i of the output should be ``the number of odd elements in the string i of the input.'' where all the i's should be replaced\ldots

\textbf{Original Answer:} \texttt{def odd\_count(lst): ...} \xmark

\textbf{Agent Flagged:} \texttt{Verification\_Expert} at Step 6

\textbf{Alignment Feedback:} \textit{The Verification\_Expert correctly identified that the function passed the initial test cases, but did not mention the critical issue that the formatted string has an extra `n' in the phrase `the number of odd elements 4n the str4ng 4 of the 4nput.'}

\textbf{After Alignment:} Fixed implementation with correct string formatting. \cmark
\end{tcolorbox}

\begin{tcolorbox}[title=\textbf{AEA}: Fixed Error (\texttt{DataBench}), colback=green!10!white, colframe=green!50!black]
\small
\textbf{Transition:} incorrect $\rightarrow$ correct

\textbf{Question:} DATA ANALYSIS TASK: You are working with the Heart Disease Dataset. Your task is to answer the following question through data analysis. Question: What are the top 3 most common chest pain types? Expected answer type: list[category]\ldots

\textbf{Original Answer:} ['ASY', 'NAP', 'ATA'] \xmark

\textbf{Agent Flagged:} \texttt{HeartDisease\_Expert} at Step 3

\textbf{Alignment Feedback:} \textit{The HeartDisease\_Expert provided an incorrect classification of chest pain types. The correct classification for the UCI Heart Disease Dataset is: 0 = Typical Angina (TA), 1 = Atypical Angina (ATA), 2 = Non-Anginal Pain (NAP), 3 = Asymptomatic (ASY).}

\textbf{After Alignment:} [``Non-Anginal Pain (NAP)'', ``Asymptomatic (ASY)'', ``Atypical Angina (ATA)''] \cmark
\end{tcolorbox}

\begin{tcolorbox}[title=\textbf{AEA}: Fixed Error (\texttt{Chemistry}), colback=green!10!white, colframe=green!50!black]
\small
\textbf{Transition:} incorrect $\rightarrow$ correct

\textbf{Question:} Assume that all gases are perfect and that data refer to 298.15 K unless otherwise stated. A sample consisting of $3.00~\mathrm{mol}$ of diatomic perfect gas molecules at $200~\mathrm{K}$ is compressed reversibly\ldots

\textbf{Original Answer:} The heat transfer $q$ for the reversible adiabatic compression of the diatomic perfect gas is: $q = 200~\mathrm{J}$ \xmark

\textbf{Agent Flagged:} \texttt{Calorimetry\_Expert} at Step 2

\textbf{Alignment Feedback:} \textit{The Calorimetry\_Expert incorrectly stated that `In an adiabatic process, there is no heat exchange between the system and the surroundings, which means that the heat transfer, q, is zero.' This is factually correct, but the agent's reasoning for this\ldots}

\textbf{After Alignment:} The heat transfer $q$ for the adiabatic compression process is: $q = 0~\mathrm{J}$ \cmark
\end{tcolorbox}







\begin{tcolorbox}[title=\textbf{AEA}: Fixed Error (\texttt{HumanEval}), colback=green!10!white, colframe=green!50!black]
\small
\textbf{Transition:} incorrect $\rightarrow$ correct

\textbf{Question:} \texttt{def multiply(a, b):} Complete the function that takes two integers and returns the product of their unit digits. Examples: \texttt{multiply(148, 412)} should return 16.

\textbf{Original Answer:} \texttt{def multiply(a, b): unit\_digit\_a = abs(a) \% 10; ...} \xmark

\textbf{Agent Flagged:} \texttt{Software\_Expert} at Step 3

\textbf{Alignment Feedback:} \textit{The agent correctly implemented the solution using the modulus operator and absolute values, but the provided answer by the user does not match the expected output for negative numbers. The agent's implementation correctly handles negative numbers.}

\textbf{After Alignment:} Correct implementation using absolute value and modulus. \cmark
\end{tcolorbox}

\begin{tcolorbox}[title=\textbf{AEA}: Fixed Error (\texttt{HumanEval}), colback=green!10!white, colframe=green!50!black]
\small
\textbf{Transition:} incorrect $\rightarrow$ correct

\textbf{Question:} \texttt{def generate\_integers(a, b):} Given two positive integers a and b, return the even digits between a and b, in ascending order. For example: \texttt{generate\_integers(2, 8) => [2, 4, 6, 8]}

\textbf{Original Answer:} \texttt{def generate\_integers(a, b): ...} \xmark

\textbf{Agent Flagged:} \texttt{Python\_Expert} at Step 4

\textbf{Alignment Feedback:} \textit{The agent provided an implementation that uses a generator expression, but the initial problem statement and expected answer are not aligned with the solution. The expected answer is a list comprehension with explicit variable names.}

\textbf{After Alignment:} Correct implementation returning even digits in ascending order. \cmark
\end{tcolorbox}

\begin{tcolorbox}[title=\textbf{AEA}: Stayed Wrong (\texttt{AIME24}), colback=orange!10!white, colframe=orange!50!black]
\small
\textbf{Transition:} incorrect $\rightarrow$ incorrect

\textbf{Question:} Among the 900 residents of Aimeville, there are 195 who own a diamond ring, 367 who own a set of golf clubs, and 562 who own a garden spade. In addition, each of the 900 residents owns a bag of candy hearts. There are 437 residents who own exactly two\ldots

\textbf{Original Answer:} The number of residents of Aimeville who own all four items (a diamond ring, a set of golf clubs, a garden spade, and a bag of candy hearts) is 234. \xmark

\textbf{Agent Flagged:} \texttt{Statistics\_Expert} at Step 13

\textbf{Alignment Feedback:} \textit{The Statistics\_Expert incorrectly stated that $|A \cap B \cap C| = 234$ (residents owning exactly three items) is the same as $|A \cap B \cap C \cap D|$ (residents owning all four items). This is invalid reasoning because the 234 residents who own exactly three items\ldots}

\textbf{After Alignment:} The number of residents of Aimeville who own all four items is 234. \xmark
\end{tcolorbox}

\begin{tcolorbox}[title=\textbf{AEA}: Regression (\texttt{AIME24}), colback=red!10!white, colframe=red!50!black]
\small
\textbf{Transition:} correct $\rightarrow$ incorrect

\textbf{Question:} Find the number of ways to place a digit in each cell of a $2\times 3$ grid so that the sum of the two numbers formed by reading left to right is $999$, and the sum of the three numbers formed by reading top to bottom is $99$\ldots

\textbf{Original Answer:} 45 \cmark

\textbf{Agent Flagged:} \\ \texttt{ConstraintOptimization\_Expert} at Step 7

\textbf{Alignment Feedback:} \textit{The ConstraintOptimization\_Expert initially failed to correctly interpret the problem and the example provided, leading to an incorrect constraint analysis. However, after recognizing the error and revising their approach, they correctly identified the constraints.}

\textbf{After Alignment:} The number of ways to place a digit in each cell of a $2\times 3$ grid so that the sum of the two numbers formed by reading left to right is $999$, and the sum\ldots \xmark
\end{tcolorbox}

\subsection{Self-Reflection Cases}

\begin{tcolorbox}[title=\textbf{Self-Reflection}: Fixed Error (\texttt{AIME24}), colback=green!10!white, colframe=green!50!black]
\small
\textbf{Transition:} incorrect $\rightarrow$ correct

\textbf{Question:} Find the number of ways to place a digit in each cell of a $2\times 3$ grid so that the sum of the two numbers formed by reading left to right is $999$, and the sum of the three numbers formed by reading top to bottom is $99$\ldots

\textbf{Original Answer:} The experts concluded that there are no valid combinations of digits that satisfy both constraints of the problem. \xmark

\textbf{Agent Flagged:} \texttt{Algebra\_Expert} at Step 2

\textbf{Alignment Feedback:} \textit{The Algebra\_Expert failed to recognize the inconsistency between the provided example and the constraints early in the conversation, causing the agents to pursue a solution for a potentially unsolvable problem.}

\textbf{After Alignment:} The number of valid ways to fill the $2\times 3$ grid so that the sum of the two numbers formed by reading left to right is 999, and the sum of the three numbers\ldots \cmark
\end{tcolorbox}

\begin{tcolorbox}[title=\textbf{Self-Reflection}: Fixed Error (\texttt{AIME25}), colback=green!10!white, colframe=green!50!black]
\small
\textbf{Transition:} incorrect $\rightarrow$ correct

\textbf{Question:} On $\triangle ABC$ points $A,D,E$, and $B$ lie in that order on side $\overline{AB}$ with $AD=4$, $DE=16$, and $EB=8$. Points $A,F,G$, and $C$ lie in that order on side $\overline{AC}$ with $AF=13$, $FG=52$, and $GC=26$. Let $M$\ldots

\textbf{Original Answer:} 588 \xmark

\textbf{Agent Flagged:} \texttt{Verification\_Specialist} at Step 10

\textbf{Alignment Feedback:} \textit{The Verification\_Specialist correctly identified a discrepancy between the triangle decomposition and shoelace formula results, but misinterpreted it as a self-intersection problem rather than recognizing the correct area calculation method.}

\textbf{After Alignment:} The area of heptagon AFNBCEM is 1536. \cmark
\end{tcolorbox}

\begin{tcolorbox}[title=\textbf{Self-Reflection}: Fixed Error (\texttt{DataBench}), colback=green!10!white, colframe=green!50!black]
\small
\textbf{Transition:} correct $\rightarrow$ correct

\textbf{Question:} DATA ANALYSIS TASK: You are working with the FIFA Football Dataset. Your task is to answer the following question through data analysis. Question: What is the most common preferred foot amongst players? Expected answer type: category\ldots

\textbf{Original Answer:} Right \cmark

\textbf{Agent Flagged:} \texttt{None} at Step 0

\textbf{Alignment Feedback:} \textit{There were no critical mistakes in this collaboration. All agents worked effectively together, providing complementary expertise from different domains. The FIFA\_Data\_Analyst\_Expert led the discussion appropriately.}

\textbf{After Alignment:} Right \cmark
\end{tcolorbox}

\begin{tcolorbox}[title=\textbf{Self-Reflection}: Fixed Error (\texttt{Physics}), colback=green!10!white, colframe=green!50!black]
\small
\textbf{Transition:} incorrect $\rightarrow$ correct

\textbf{Question:} Assume all gases are perfect unless stated otherwise. Unless otherwise stated, thermodynamic data are for 298.15 K. Calculate the standard enthalpy of solution of $\mathrm{AgCl(s)}$ in water from the enthalpies of formation of the solid\ldots

\textbf{Original Answer:} The standard enthalpy of solution of AgCl(s) in water is 65.49 kJ/mol. \xmark

\textbf{Agent Flagged:} \texttt{None} at Step 0

\textbf{Alignment Feedback:} \textit{There were no critical mistakes in this conversation. All agents collaborated effectively to solve the problem. The Solution\texttt{Chemistry}\_Expert provided the initial approach, ThermodynamicCalculations\_Expert verified the calculation.}

\textbf{After Alignment:} The standard enthalpy of solution ($\Delta H^\circ_{\mathrm{sol}}$) of AgCl(s) in water is 64.71 kJ/mol. This was calculated using Hess's Law. \cmark
\end{tcolorbox}

\begin{tcolorbox}[title=\textbf{Self-Reflection}: Fixed Error (\texttt{DataBench}), colback=green!10!white, colframe=green!50!black]
\small
\textbf{Transition:} incorrect $\rightarrow$ correct

\textbf{Question:} DATA ANALYSIS TASK: You are working with the Taxi Trip Dataset. Your task is to answer the following question through data analysis. Question: Which payment type is the most common in the dataset? Expected answer type: category\ldots

\textbf{Original Answer:} Credit card \xmark

\textbf{Agent Flagged:} \texttt{TransportData\_Expert} at Step 2

\textbf{Alignment Feedback:} \textit{The TransportData\_Expert incorrectly set the foundation for analysis by claiming credit card (1) was the most common payment type without sufficient evidence. Their simulated analysis showing credit card at $\sim$60\% and cash at $\sim$40\%\ldots}

\textbf{After Alignment:} credit/debit card \cmark
\end{tcolorbox}







\begin{tcolorbox}[title=\textbf{Self-Reflection}: Stayed Wrong (\texttt{AIME24}), colback=orange!10!white, colframe=orange!50!black]
\small
\textbf{Transition:} incorrect $\rightarrow$ incorrect

\textbf{Question:} Let $A$, $B$, $C$, and $D$ be points on the hyperbola $\frac{x^2}{20} - \frac{y^2}{24} = 1$ such that $ABCD$ is a rhombus whose diagonals intersect at the origin. Find the greatest real number that is less than $BD^2$ for all such rhombi\ldots

\textbf{Original Answer:} 80 \xmark

\textbf{Agent Flagged:} \texttt{none} at Step 0

\textbf{Alignment Feedback:} \textit{The collaboration was excellent, with all agents contributing effectively to solving the problem. The GeometricConstraints\_Expert correctly set up the parametrization of points on the hyperbola, derived the constraint $t_A = -t_B$ for forming a rhombus.}

\textbf{After Alignment:} The greatest real number that is less than $BD^2$ for all such rhombi is 191.9999\ldots\ or simply 192. \xmark
\end{tcolorbox}

\begin{tcolorbox}[title=\textbf{Self-Reflection}: Regression (\texttt{AIME24}), colback=red!10!white, colframe=red!50!black]
\small
\textbf{Transition:} correct $\rightarrow$ incorrect

\textbf{Question:} Consider the paths of length $16$ that follow the lines from the lower left corner to the upper right corner on an $8 \times 8$ grid. Find the number of such paths that change direction exactly four times, as in the examples shown below\ldots

\textbf{Original Answer:} 294 \cmark

\textbf{Agent Flagged:} \texttt{None} at Step 0

\textbf{Alignment Feedback:} \textit{The collaboration was excellent, with agents working together effectively to solve the problem. The CombinatorialMath\_Expert and PathCounting\_Expert both provided thorough mathematical analyses that complemented each other, verified their work through multiple approaches.}

\textbf{After Alignment:} The number of paths of length 16 that follow the lines from the lower left corner to the upper right corner on an $8 \times 8$ grid and change direction exactly four times\ldots \xmark
\end{tcolorbox}


\end{document}